\documentclass[a4paper]{paper}

\usepackage[utf8]{inputenc} 
\usepackage[T1]{fontenc}  
\usepackage{graphicx}
\usepackage{pgfplots}
\pgfplotsset{compat=1.17}
\usetikzlibrary{spy}
\usepackage{appendix}
\usepackage{amsmath}
\usepackage{amsthm}
\usepackage{amssymb}    
\usepackage{mathtools}  
\usepackage{cases}
\usepackage{tikz-cd}
\usepackage{lmodern}   
\usepackage{microtype}     
\usepackage{amsfonts}
\usepackage{mathrsfs}
\usepackage{paralist}
\usepackage[giveninits=true]{biblatex} 
\usepackage[textfont=normalfont,singlelinecheck=off,justification=raggedright]{subcaption}
\usepackage{hyperref}  
\hypersetup{%
    pdfmenubar=true,       
    pdffitwindow=false,     
    pdfstartview={FitH},    
    colorlinks=true,       
    linkcolor=blue,          
    citecolor=red,        
    filecolor=blue,      
    urlcolor=blue,           
}
\usepackage{cleveref}      
\usepackage{verbatim}
\usepackage{acro}
\usepackage{paperacronyms}
\definecolor{OColor}{rgb}{1.0, 0.65, 0.0}
\definecolor{JColor}{rgb}{0.0, 0.75, 1.0}
\definecolor{TColor}{rgb}{1.0, 0.27, 0.0}
\usepackage[textsize=tiny]{todonotes}

\makeatletter
\providecommand{\maketitle}{}
\renewcommand{\maketitle}{%
	\par
	\begingroup
	\renewcommand{\thefootnote}{\fnsymbol{footnote}}
	\renewcommand{\@makefnmark}{\hbox to \z@{$^{\@thefnmark}$\hss}}
	\long\def\@makefntext##1{%
		\parindent 1em\noindent
		\hbox to 1.8em{\hss $\m@th ^{\@thefnmark}$}##1
	}
	\thispagestyle{empty}
	\@maketitle
	\@thanks
	\endgroup
	\let\maketitle\relax
	\let\thanks\relax
}

\providecommand{\@maketitle}{}
\renewcommand{\@maketitle}{%
	\vbox{%
		\hsize\textwidth
		\linewidth\hsize
		\vskip 0.1in
		\centering
		{\LARGE\bf \@title\par}
		\def\And{%
		\end{tabular}\hfil\linebreak[0]\hfil%
		\begin{tabular}[t]{c}\bf\rule{\z@}{24\p@}\ignorespaces%
		}
		\def\AND{%
		\end{tabular}\hfil\linebreak[4]\hfil%
		\begin{tabular}[t]{c}\bf\rule{\z@}{24\p@}\ignorespaces%
		}
		\begin{tabular}[t]{c}\bf\rule{\z@}{24\p@}\@author\end{tabular}%
		\vskip 0.3in \@minus 0.1in
	}
}
\makeatother

\newcommand{\Real}{\mathbb{R}}

\newcommand{\RecSpace}{X}
\newcommand{\DataSpace}{Y}

\newcommand{\PClass}[1]{\mathscr{P}_{#1}}
\newcommand{\stochastic}[1]{\mathsf{#1}}
\newcommand{\PDens}{\mathsf{P}}
\newcommand{\Expect}{\operatorname{\mathbb{E}}}
\newcommand{\pdf}{\rho}
\newcommand{\PDist}{\mathcal{W}}
\DeclareMathOperator{\KL}{KL}
\DeclareMathOperator{\ELBO}{ELBO}
\newcommand{\Op}[1]{\operatorname{\mathcal{#1}}}
\newcommand{\ForwardOp}{\Op{A}}
\DeclareMathOperator*{\argmin}{arg\,min}
\DeclareMathOperator*{\argmax}{arg\,max}

\newcommand{\RegFunc}{\Op{R}}
\newcommand{\DataLogLikelihood}{\Op{L}}
\newcommand{\signaldim}{n}
\newcommand{\signal}{x}
\newcommand{\signaltrue}{\signal^*}

\newcommand{\datadim}{l}
\newcommand{\data}{y}

\newcommand{\datanoise}{e}

\newcommand{\regparamrec}{\lambda}
\newcommand{\SignalPrior}{\PDens_{\RecSpace}}
\newcommand{\ModelDistr}{\mathsf{Q}}
\newcommand{\stsignal}{\stochastic{\signal}}
\newcommand{\stsignalnoise}
  {\triangle \stochastic{\signal}}
\newcommand{\signalnoisepar}{\sigma}
\newcommand{\stdata}{\stochastic{\data}}
\newcommand{\stdatanoise}{\stochastic{\datanoise}}

\newcommand{\PriorDictCoeff}{\PDens_{\DictCoeffSet}}
\newcommand{\PriorNoise}{\PDens_{\stsignalnoise}}
 
\newcommand{\DictClass}{\mathscr{D}}
\newcommand{\Dict}{\mathcal{D}}
\newcommand{\DictMat}{\mathbf{D}}

\newcommand{\dictatom}{d}

\newcommand{\dictcoeffdim}{m}
\newcommand{\dictcoeff}{z}

\newcommand{\dictcoeffpar}{b}
\newcommand{\dictcoeffmode}{{z^*}}
\newcommand{\dictcoeffmodepar}{{b^*}}
\newcommand{\stdictcoeff}{\stochastic{\dictcoeff}}

\newcommand{\DictCoeffSet}{Z}
\newcommand{\SynthesisOp}{\operatorname{\mathcal{S}}}

\newcommand{\TData}{\operatorname{\Sigma}}
\newcommand{\ApproxPdf}{q}
\newcommand{\ExpectApproxPdf}{\Expect_{\ApproxPdf} }
\newcommand{\ApproxStCoef}{\Tilde{\stdictcoeff}}
\newcommand{\estim}[1]{\widehat{#1}}

\DeclareMathOperator{\dint}{d\!}
\DeclareMathOperator{\tr}{tr}

\newtheorem{theorem}{Theorem}[section]
\usepackage{amsmath}               
  {
      \theoremstyle{plain}
      \newtheorem{assumption}{Assumption}[section]
  }
\usepackage{enumitem}
\newlist{enumerateA}{enumerate}{1}
\setlist[enumerateA]{label=(A\arabic*)}

\title{Deep learning based dictionary learning and tomographic image reconstruction}
\author{Jevgenija Rudzusika\vspace{0.5mm} \\
  Department of Mathematics \\ 
  KTH Royal Institute of Technology\\
  \href{mailto:jevaks@kth.se}{\texttt{jevaks@kth.se}} 
  \And Thomas Koehler \vspace{0.5mm} \\
  Philips Research \\
  \href{mailto:thomas.koehler@philips.com}{\texttt{thomas.koehler@philips.com}}
  \And Ozan \"Oktem \vspace{0.5mm} \\
  Department of Mathematics \\ 
  KTH Royal Institute of Technology\\
  \href{mailto:ozan@kth.se}{\texttt{ozan@kth.se}}}
\date{}

\begin{document}
\maketitle
\begin{abstract}

This work presents an approach for image reconstruction in clinical low-dose tomography that combines principles from sparse signal processing with ideas from deep learning.  First, we describe sparse signal representation in terms of dictionaries from a statistical perspective and interpret dictionary learning as a process of aligning distribution that arises from a generative model with empirical distribution of true signals. As a result we can see that sparse coding with learned dictionaries resembles a specific variational autoencoder, where the decoder is a linear function and the encoder is a sparse coding algorithm. Next, we show that dictionary learning can also benefit from computational advancements introduced in the context of deep learning, such as parallelism and as stochastic optimization.
Finally, we show that regularization by dictionaries achieves competitive performance in \ac{CT} reconstruction comparing to state-of-the-art model based and data driven approaches.
\end{abstract}

\section{Introduction}
This work presents an approach for image reconstruction in clinical low-dose tomography that combines principles from sparse signal processing with ideas from deep learning. 
Before describing the method, we begin with providing some background in this section that serves as a motivation for our work.  

\subsection{X-ray tomographic imaging in medicine}
In \ac{CT} an x-ray tube and a detector rotate around a patient, acquiring x-ray transmission measurements from multiple directions.
The aim is to computationally recover an image showing the interior anatomy of the patient. 
This is an inherently unstable procedure, so a key issue lies in adequately addressing this instability.

\Ac{CT} is nowadays one of the most frequently used imaging modalities and approximately 100 million \ac{CT} scans are being performed annually in the United States alone.
The repeated radiation from multiple directions means exposing the patient to ionising x-rays. 
Rising concerns about radiation dose in clinical \ac{CT} imaging lead to increasing interest in low-dose protocols \cite{Brenner:2007aa,Tack:2007aa,Sarma:2012aa}.
One approach to lower the dose is to  collect data along fewer directions (sparse-view \ac{CT}). Alternatively, one can keep full sampling of data and instead reduce the intensity of emitted x-ray photons, e.g., by lowering the tube current and/or voltage in the scanner, which in turn results in data that is significantly more noisy.
This second approach is not only more practical, it is also more efficient, i.e., one can obtain better image quality from data with the same relative dose reduction \cite{Kalra:2004aa,McCollough:2009aa}.
This is true at least for the standard baseline methods. Therefore, all of our low-dose \ac{CT} data refers to data from full angular sampling, but with high noise. 

The significance of how the aforementioned instability is addressed during reconstruction increases with the noise level in data. 
In particular, it is evident that the computationally efficient reconstruction methods originally developed for image reconstruction from less noisy normal-dose \ac{CT} data are inadequate for low-dose \ac{CT}. 
The resulting images are degraded by noise and artefacts, which in turn renders them sub-optimal for diagnostic interpretation.
Addressing this has been an active area for research as briefly outlined in  \cref{sec:LowDoseCT}.

\subsection{Inverse problems and regularization}
\Ac{CT} imaging is an example of an inverse problem. 
The latter refers to problems where the goal is to recover a hidden signal from measured data that represent noisy indirect observations of the signal. 
Such problems arise in many areas of science and engineering.

Many inverse problems are large-scale in the sense that data and/or signal reside in high-dimensional spaces even after clever discretization.
As an example, signals and data in 2D/3D tomographic imaging are represented by high dimensional arrays.
Solution methods must therefore have computationally feasible implementations, which is especially important in time critical applications.

A key challenge in solving inverse problems is to address \emph{instability (ill-posedness)}. 
This refers to a situation where merely maximising the fit against measured data is an unstable procedure, i.e., small errors in data result in large perturbations to the resulting signal.
Solving an ill-posed inverse problem therefore requires specific attention to handle the intrinsic instability.
\emph{Regularization} refers to mathematical theory and algorithms that introduce stability by balancing the need to fit data against having a reconstruction that is consistent with known information about the unknown signal one seeks to recover (prior model).
The data misfit can be quantified by the \emph{data (log) likelihood}, which consists of a \emph{forward operator} (models how a signal gives rise to data in absence of noise) and a \emph{noise model} (encodes statistical properties of the observation errors). 
Both these constitute a simulator and they are typically derived from first principles by carefully considering the physics that governs the formation of data.  
Prior models are, in contrast to the data likelihood, not derived from the physics of data acquisition.
Next, the impact from a specific choice of prior becomes increasingly notable as the noise level in data increases.
Hence, a major topic during the last two decades in inverse problems research has been to devise appropriate prior models. Early approaches had smoothing priors for recovering low-frequency components of the signal. These were followed by more intricate regularity priors, like more domain adapted sparsity promoting priors that are either handcrafted or learned from example data.

Overall, it is still challenging to suggest a reconstruction method for low-dose \ac{CT} that is \emph{computational feasible}, yet is based on a  \emph{prior model} that provides a clinically notable improvement in image quality.

\subsection{Priors given by generative models}

A common approach for assembling a domain adapted prior model is to start from a \emph{generative model} that is defined by a \emph{synthesis operator} $\SynthesisOp \colon \DictCoeffSet \to \RecSpace$, which generates a signal in $\RecSpace$ from a representation in \emph{latent space} $\DictCoeffSet$. 
This representation is often not unique, therefore a criterion for choosing most suitable representation is introduced. 
This criterion can usually be expressed as minimising a (regularization) functional $\RegFunc \colon \DictCoeffSet \to \Real$. The generative model and the aforementioned regularization functional define a prior model, which can be used to address the inherit instability in solving an ill-posed inverse problem as outlined in~\cref{sec:ct_rec}.

\paragraph{Dictionaries and sparsity}
A sparsity prior  assumes that signals from some vector space $\RecSpace$ can be described as a \emph{linear combinations} of a \emph{few} signal components (atoms) from a pre-specified \emph{dictionary}. 

Stated formally, a \emph{dictionary} is some countable subset $\Dict = \{ \dictatom_i \}_i \subset \RecSpace$ whose elements $\dictatom_i$ are called \emph{(dictionary) atoms}.
The synthesis operator defining the generative model assembles signals in $\RecSpace$ by taking linear combinations of atoms:
\begin{equation}\label{eq:SynthesisOp}
  \SynthesisOp_{\Dict}(\dictcoeff) =
     \sum_i \dictcoeff_i \dictatom_i
  \quad\text{for $\dictcoeff_i \in \Real$ and $\dictatom_i \in \Dict \subset \RecSpace$.}
\end{equation}
The latent variables $\dictcoeff_i \in \Real$ are called \emph{dictionary coefficients}. 

A sparsity prior implies that only few non-zero coefficients are sufficient to represent a signal.
Unfortunately, the generative model in \cref{eq:SynthesisOp} in combination with a sparsity prior would require too many dictionary atoms to represent large and complex signals like images.
Therefore, it is typically applied to smaller image patches and representation of the full image is obtained by averaging representations of overlapping patches. An alternative is to consider convolutional dictionaries:
\begin{equation}\label{eq:SynthesisOpConv}
  \SynthesisOp_{\Dict}(\dictcoeff) =
     \sum_i \dictcoeff_i \ast \dictatom_i
  \quad\text{for $\dictcoeff_i \in \RecSpace$ and $\dictatom_i \in \Dict \subset \RecSpace$.}
\end{equation}
The dictionary atoms $\dictatom_i$ represent small local features, i.e.\ have a relatively small support compared to signals in $\RecSpace$.
They synthesise a signal by convolving against (local) signal maps, so dictionary coefficients $\dictcoeff_i$ are elements in $\RecSpace$ (coefficient maps) instead of simply real numbers.

The problem of computing sparse dictionary coefficients for a given signal $\signal$ is called sparse coding. It is natural to define this problem as finding coefficients $\estim{\dictcoeff}$ that minimize the objective:
\begin{equation}\label{eq:SparseCoding0}
    \displaystyle{\min_{\dictcoeff \in \DictCoeffSet}} 
    \bigl\| \SynthesisOp_{\Dict}(\dictcoeff) - \signal \bigr\|^2_{\RecSpace} + \regparamrec  \|\dictcoeff \|_0.
\end{equation}
The $l_0$ pseudo-norm in $\| \dictcoeff \|_0$ simply counts the number of non-zero terms in the sequence $\dictcoeff$. Since each non-zero term corresponds to an atom in $\Dict$, minimizing the objective means we look for a sparse representation (the one with few number of atoms), which sufficiently well approximates the original signal. 
In particular, $\regparamrec \geq 0$ is a weighting factor between sparsity and how well the coding matches the signal. A small value means the coding has to be close to the signal at the expense of sparsity, whereas increasing $\regparamrec$ means prioritising sparsity over the need to match the signal.

Unfortunately, the presence of an $l_0$-pseudo norm in \cref{eq:SparseCoding0} makes the optimization problem hard to solve. Indeed, it is known that this problem is NP-hard \cite{AmaldiKann1998,Tillmann:2015aa}. This difficulty is addressed by using approximation algorithms. There are generally two types of approximation techniques for this purpose. 
Greedy algorithms, like orthogonal matching pursuit \cite{tropp2007signal} and iterative hard thresholding \cite{blumensath2008iterative}, address the problem \cref{eq:SparseCoding0} directly, while methods based on convex relaxation replace the $l_0$-pseudo norm with an $l_1$-norm \cite{candes2006robust}:
\begin{equation}\label{eq:SparseCoding1}
    \min_{\dictcoeff \in \DictCoeffSet}
    \bigl\| \SynthesisOp_{\Dict}(\dictcoeff) - \signal \bigr\|^2_{\RecSpace} + \regparamrec  \|\dictcoeff \|_1.
\end{equation}
In general, a solution to \cref{eq:SparseCoding0} yields representations that can be more sparse than the representation obtained from solving \cref{eq:SparseCoding1}.
However, if $\signal$ admits a sufficiently sparse representation in $\Dict$ and $\Dict$ satisfies the restricted isometry property, then a solution to \cref{eq:SparseCoding1} is with high probability also a solution to \cref{eq:SparseCoding0} \cite{Donoho2006,CandesRombergTao2006,CandesTao2006}.

\paragraph{Dictionary learning}

A good sparsifying dictionary will generate sparse codes of signals that preserve important features, whereas noise and artefacts are preferably suppressed (or even lost).
The choice of dictionary is therefore an essential component, and in general, this can be done using one of two ways: 
\begin{inparaenum}[(a)]
\item building a dictionary based on a mathematical model of signals in $\RecSpace$, or 
\item learning a dictionary to perform best on a training set \cite{Rubinstein:2010aa}.
\end{inparaenum}
Dictionary learning focuses on the latter, i.e., on the task of learning an appropriate sparse representation from example data.
This is an inverse problem whose solution is a (trained) dictionary, which in turn defines a prior model on $\RecSpace$.

Let $\signal_1, \ldots, \signal_n \in \RecSpace$ be the example signals and $\DictClass$ is some fixed family of dictionaries (countable subsets of $\RecSpace$).
A common approach is to perform dictionary learning jointly with sparse coding, e.g., by minimizing the following objective with respect to both, dictionaries $\Dict \in \DictClass$ and coefficients $\dictcoeff_i \in \DictCoeffSet$:
\begin{equation}\label{eq:typical}
\begin{split}
    \min_{\Dict,\dictcoeff_i} &
    \sum_{i=1}^n \bigl\|
      \SynthesisOp_{\Dict}(\dictcoeff_i) - \signal_i 
     \bigr\|_{\RecSpace}^2 
     + \regparamrec \|\dictcoeff_i\|_1, \\
     \text{s.t. } &\|\dictatom_j\| = 1 \quad \text{for all } \dictatom_j \in \Dict.  
\end{split}
\end{equation}
The constraint on the norm of dictionary atoms $\|\dictatom_j\| = 1$ is included to circumvent the fact that coefficients $\dictcoeff_i$ can be reduced simply by up-scaling the corresponding dictionaries.

In \cref{sec:dl_theory} we present a statistical formulation of dictionary representation that provides a mathematical interpretation of the dictionary learning scheme in \cref{eq:typical}. 

\paragraph{Dictionary based \ac{CT} reconstruction}
Our goal here is to solve the inverse problem arising in  \ac{CT} imaging (\emph{image reconstruction}) assuming the true (unknown) image is generated by the aforementioned generative model.
This assumption may serve as a regularization, which is in particular the case when the generative model is not capable of generating undesirable noise and artefacts.

These two inverse problems, i.e., defining the generative model and image reconstruction, can be solved sequentially or jointly. The former means that dictionary learning is performed without considering the \ac{CT} image reconstruction step that follows next. 

In this work, we try to benefit from computational advancements of deep learning while staying in a highly interpretative framework of representation through dictionaries.
In this respect our work is strongly related to \cite{tolooshams2020deep}. 
However, our main focus is on dictionary learning as a regularization method in inverse problems, specifically in \ac{CT}.

\subsection{Outline of paper}

In the next section (\cref{sec:survey}) we present a survey of related works on dictionary learning and reconstruction in computerized tomography.
\Cref{sec:dl_theory} presents dictionary learning from a statistical perspective as a way to approximate an empirical distribution of a natural signal. 
Secondly, \cref{ssec:dl_implementation} describes a practical implementation of the learning procedure.
This is based on the convolutional generative model in \cref{eq:SynthesisOpConv} and the sparse coding in \cref{eq:SparseCoding1} that defines a prior model parametrised by a learned dictionary.
Next, in \cref{sec:dl_regularization} we describe how the learned dictionary model can be used for regularization in inverse problems. In particular, \cref{sec:ct_rec} describes  our reconstruction experiments in computerized tomography. 
\Cref{sec:conclusion} concludes the paper.

\section{Survey of related work}
\label{sec:survey}
\subsection{Image reconstruction in clinical low-dose \ac{CT}}\label{sec:LowDoseCT}
Image reconstruction for clinical \ac{CT} has traditionally relied on the \ac{FBP} method and variants thereof.
These analytical methods were first in introduced in late 1960s in astronomy \cite{Bracewell:1967aa} and later adopted by the medical imaging community in the early 1970s \cite{Shepp:1974aa}.
They use principles from Fourier analysis and sampling theory to recover the part of the image that is band-limited whereas high frequency components, like noise, are filtered out.
The mathematical foundation of \ac{FBP} was developed in the late 1970s \cite{Natterer:1980aa} and it has since then continuously evolved to account for increasingly complex acquisition geometries, like those that arise in 3D clinical \ac{CT} \cite{Feldkamp:1984aa,Faber:1995aa,Katsevich:2003aa,Noo:2003aa,Kohler:2006aa,Zhu:2004aa,Katsevich:2006aa,Yu:2006aa}.
The compromise that \ac{FBP} type of methods make in balancing image quality against reconstruction speed is still difficult to outperform in the context of clinical normal dose \ac{CT} imaging \cite{Pan:2009aa}.

\Ac{FBP} type of methods do not handle the noise statistics of measured data optimally. They regularise by recovering the band-limited part of the image, an approach that is not sufficient to suppress noise and artefacts that degrade image quality in low-dose \ac{CT}. 
Hence, \ac{FBP} type of methods render images in low-dose \ac{CT} that are sub-optimal for diagnostic interpretation.
The need to address this issue has catalysed the development of new iterative reconstruction algorithms in both academia and industry. 
Only a fraction of the methods developed in academia for low-dose \ac{CT} have made it into the clinic practice, and this almost always as part of collaboration with vendors of \ac{CT} scanners as surveyed next.

The first commercially available iterative reconstruction algorithms for replacing \ac{FBP} appeared in 2009 with the launch of IRIS (Siemens Healthineers) and ASiR (GE~Healthcare).
Since then, within a few years, all major \ac{CT} vendors introduced iterative reconstruction algorithms for clinical routine. 
Examples are AIDR (Canon Medical Systems, 2010) and AIDR3D (Canon Medical Systems, 2012), iDose${}^{4}$ (Philips Healthcare, 2012), SAFIRE (Siemens Healthineers, 2011) and ADMIRE (Siemens Healthineers, 2014).
These methods (often referred to as hybrid techniques or statistical iterative reconstruction) define an iterative scheme that combines denoising data and/or image with \ac{FBP} to map data into image space. 
Variational models (also called model-based iterative reconstruction) were introduced somewhat later. 
These come with rigorous mathematical foundations and in addition to noise and photon statistics, one can also model object, scanner geometry and detector response.
These are better at reducing noise and artefacts than hybrid techniques, but they may also alter image texture more.
Examples are VEO (GE~Healthcare, 2011),  ASiR-V (GE~Healthcare, 2015), IMR (Philips Healthcare, 2015), and FIRST (Canon Medical Systems, 2016).
See \cite{Geyer:2015aa,Andersen:2018aa,Mileto:2019aa,Willemink:2019aa} for a review of these methods in low-dose setting.

The final line of development is the recent commercial usage of deep learning based approaches for denoising/image reconstruction, like TrueFidelity (GE~Healthcare, 2018) and AiCE (Canon Medical Systems, 2018).

\subsection{Dictionary learning}
Traditional dictionary learning typically uses a variational model with a sparsity promoting regulariser, as in  \cref{eq:typical}, to recover the dictionary and corresponding dictionary coefficients for signals in a given training data set. 

For a long time K-SVD algorithm \cite{aharon2006k} was among the best approaches for dictionary learning. 
This method updates dictionary atoms one by one, while using all the available training data. 
Such an algorithm is therefore not suitable for large training data-sets that are available nowadays. 
Convex relaxation is another, theoretically appealing approach for learning convolutional dictionaries.
It was proposed in \cite{chambolle2020convex} and the advantage is that a unique solution to the relaxed problem can be determined. Unfortunately, it comes at a cost of greatly increasing the number of optimization variables, so such an approach is suitable only for learning a few relatively smalls atoms. 

\paragraph{Statistical formulation}
An alternative viewpoint is to phrase dictionary learning as density estimation. Dictionaries are here used to construct an estimate from training data of the unobservable probability density of signals.

Such an approach taken in \cref{eq:LearningProb}, which uses variational inference to minimize the \ac{KL} divergence between the empirical distribution of signals in the training data and the distribution modeled by the dictionaries.
As shown in \cref{eq:marginal}, this naturally leads to a formulation where the joint probability density of signals and  dictionary coefficients,   $\pdf_{\stsignal,\stdictcoeff}(\signal,\dictcoeff)$, is marginalized over the dictionary coefficients.
This is computationally intractable for imaging applications. 
One option is to approximate the $\DictCoeffSet$-integral in \cref{eq:marginal} by the value of the integrand at the mode \cite{olshausen1997sparse,tolooshams2020deep}, which in turn leads to the common formulation for dictionary learning given in \cref{eq:DL_inpractice}.
However, the above cited works do not characterize the resulting  approximation error.

Another option is to approximate the integrand in \cref{eq:marginal} by un-normalized Gaussian density around the mode \cite{lewicki1999probabilistic,lewicki2000learning}. 
This is equivalent to approximating the posterior density  $\pdf_{\stdictcoeff \mid \stsignal}(\dictcoeff \mid \signal)$ by a Gaussian distribution and it leads to an optimization problem different than the one in \cref{eq:DL_inpractice}. 
Yet another dictionary learning algorithm is derived in \cite{girolami2001variational}. Here, the authors start out from expressing the Laplace prior $\pdf_{\stdictcoeff}(\dictcoeff)$ in \cref{eq:marginal} as a supremum of Gaussian densities. Then, they optimize a lower bound of the log-likelihood. 
Unfortunately, the derivation of the algorithms in these cited papers involve inversion of prohibitively large matrices that require further approximations for their proper handling. 

Our aim with \cref{sec:dl_theory} is not primarily to derive a new update rule for dictionary learning as in \cite{lewicki1999probabilistic,lewicki2000learning,girolami2001variational}. It is rather to provide a statistical interpretation of  \cref{eq:DL_inpractice} that, unlike \cite{olshausen1997sparse,tolooshams2020deep}, also includes characterizing the approximation error.

\paragraph{Connection to deep learning}
There are many attempts at connecting sparse signal processing to deep learning. 
A popular line of investigation focuses on the connection between sparse representation through convolutional dictionaries and convolutional neural networks.

More precisely, the authors of \cite{papyan2017convolutional} show that a trained convolutional neural network corresponds to an approximate algorithm for the sparse coding in a sparse multi-layer convolutional dictionary model.
This model exploits a structure given by an ordered sequence hierarchically arranged dictionaries.
Starting from the signal, atoms in one dictionary are sparsely represented by atoms in the next dictionary. 
Thus, each subsequent dictionary layer represents features with higher level of abstraction and \cite{sulam2018multilayer,sulam2019multi,mahdizadehaghdam2019deep} proposes training strategies for such models. 
The obtained sparse representations with respect to the last dictionary layer are used as input for image classification.

In \cite{mahdizadehaghdam2019deep}, it is empirically demonstrated that classification results based on such a trained sparse convolutional dictionary model are less susceptible to adversarial image perturbations than corresponding outcomes from a trained convolutional neural network.

\subsection{Dictionary based \ac{CT} reconstruction}

Several authors attempted regularization by dictionary learning in \ac{CT}. Early work used dictionary representation of image patches as a regularizing component in statistical iterative reconstruction (SIR) \cite{xu2012low}. The authors evaluated a dictionary learned from training images (GDSIR) and an adaptive dictionary learned simultaneously with reconstruction (ADSIR) and found that both are similarly effective. 
GDSIR and ADSIR set the baseline for many works to follow. In this work we use a similar problem formulation. However, the way we address practical aspects of dictionary learning and sparse coding is different.

Later, in \cite{chen2014artifact} one observed that artifacts in low-dose \ac{CT} images might give rise to sufficiently large coefficients in dictionary representation and those high coefficients won't be suppressed by thresholding operation, unless a very high threshold is set and an image becomes over-smoothed. They use carefully selected samples to train dictionaries that represent artifacts as well as tissue features 
and then set coefficients corresponding to artifacts to 0. They evaluated their method on real data and report benefits of their approach in qualitative assessment performed by radiologists. Nevertheless, the approach falls into category of post-processing reconstructions obtained by \ac{FBP} \cite{KakSlaney,WubbelingNatterer}, while we perform reconstruction and de-noising jointly by solving a variational problem.

Other attempts to improve the performance of dictionary learning include: using  $l_1$ for misfit between image and its dictionary representation \cite{zhang2016low}, smoothing intermediate image updates to remove artifacts \cite{komolafe2020smoothed},
using dictionary learning in combination with \ac{TV} \cite{zhao2019low}
and clustering patches and learning dictionaries for each class separately \cite{kamoshita2019low}.

Following the success of convolutional filters in deep learning, regularization by learned convolutional dictionaries has been applied to \ac{CT}. In \cite{bao2019convolutional}, one learns 32 convolutional dictionaries of size $10 \times 10$, while using total variation type penalty on coefficient maps. 

Furthermore, most of the previously published work use relatively small dictionaries. A common setup is to learn 256 dictionaries of size $8 \times 8$ with 5 to 10 non-zero coefficients as it was done in \cite{xu2012low}. One exception is the work in \cite{hu2016image}, where authors experimented with different patch sizes and found that the optimal patch size is around $16\times16$, with larger patches leading to slightly worse performance. 
In this work we also aim to learn larger dictionaries that can capture more information and potentially provide stronger regularization.

\subsection{Deep Learning methods for \ac{CT} image reconstruction}
Most recently, the idea of non-linear synthesis operators was introduced.
In \cite{obmann2020deep}, two networks, one for encoding and one for decoding, were trained, forcing encoders output to be sparse. Then, the decoder was used as a non-linear synthesis operator to solve regularized optimization problem. 
Similarly, in \cite{baguer2020computed} convolutional neural network was used to generate images. Network parameters were trained in unsupervised manner, during the reconstruction itself. This approach called Deep Image Prior \cite{ulyanov2018deep} is based on the idea that convolutional neural networks learn to represent low frequency components first and therefore, early stopping can be used to avoid noise. In contrast, we stick to the more traditional approach having linear synthesis operator. 

Despite the lack of inseparability and theoretical guarantees, supervised deep learning methods have been shown to produce state of the art results in \ac{CT} reconstruction. Although all of these methods share the philosophy of learning from data, there is a lot of diversity in the proposed model architectures and training procedures. These methods include Adversarial Regularizer \cite{lunz2018adversarial}, U-Net post-processing \cite{jin2017deep}, and networks for reconstruction that are trained in a supervised manner with architectures derived from unrolling iterations of a suitable optimization scheme: LEARN \cite{chen2018learn}, Learned Primal-Dual \cite{adler2018learned}, and Total Deep Variation \cite{kobler2020total}. We refer to \cite{ravishankar2019image} for a survey of the recent methods.

\section{Statistical interpretation of dictionary learning} \label{sec:dl_theory}

To formulate dictionary learning in a statistical setting, we start by considering a \emph{generative model} for signals in $\RecSpace$.
Let $\stsignal \sim \SignalPrior$ be a $\RecSpace$-valued random variable generating natural signal, e.g., 2D/3D images.
The generative model parametrised by dictionary $\Dict \subset \RecSpace$ is defined as a signal generated by sampling from the $\RecSpace$-valued random variable 
\begin{equation}\label{eq:GenModel} 
    \SynthesisOp_{\Dict}(\stdictcoeff) + \stsignalnoise
    \quad\text{given a synthesis operator $\SynthesisOp_{\Dict} \colon \DictCoeffSet \to \RecSpace$.}
\end{equation}
In the above, $\stdictcoeff \sim \PriorDictCoeff$ is a $\DictCoeffSet$-valued random variable generating dictionary coefficients and $\stsignalnoise \sim \PriorNoise$ is a $\RecSpace$-valued random variable representing random variations that arise from limitations in the model capacity of the generative model given by the dictionary.

A well-chosen generative model needs to have sufficiently high model capacity in order to represent important features of the true signal, like edges and texture in an image at various scales.  
On the other hand, a too large model capacity will also represent noise and other unwanted features.

\subsection{Dictionary learning as evidence maximisation}\label{sec:EvidenceMax}

\emph{Dictionary learning} is defined here as the task of choosing dictionary $\Dict \subset \RecSpace$ so that the statistical distribution of generated signals is as close as possible to the true (unknown) distribution of the signal.
More precisely, let $\ModelDistr_{\Dict}$ denote the distribution of signals generated as in \cref{eq:GenModel}, i.e.,  $\SynthesisOp_{\Dict}(\stdictcoeff) + \stsignalnoise \sim \ModelDistr_{\Dict}$.
We then seek the dictionary that maximises the match to the true distribution $\SignalPrior$, e.g., by solving 
\begin{equation}\label{eq:LearningProb}
  \estim{\Dict} 
    \in \argmin_{\Dict} \PDist\bigl( \SignalPrior,  \ModelDistr_{\Dict}\bigr).
\end{equation}
Here, $\PDist$
quantifies similarity between probability distributions on $\RecSpace$.
The above essentially amounts to choosing the dictionary $\Dict$ so that $\SynthesisOp_{\Dict}(\stdictcoeff) + \stsignalnoise$ is as close as possible to $\stsignal$ as random variables in $\RecSpace$.
The optimisation in \cref{eq:LearningProb} can be non-convex, which is why we define $\estim{\Dict}$ with `$\in$' instead of equality.

An issue with the formulation in \cref{eq:LearningProb} is that it assumes one has access to the true (unknown) signal distribution $\SignalPrior$.
Often, one has access to i.i.d.\ samples $\TData := \{ \signal_1, \ldots, \signal_N \} \subset \RecSpace$ generated by $\stsignal \sim \SignalPrior$.
One can replace the unknown distribution $\SignalPrior$ in \cref{eq:LearningProb} with $\estim{\PDens}_{\TData} \in \PClass{\RecSpace}$, which is the empirical measure given by the aforementioned training data, so \cref{eq:LearningProb} is replaced with 
\begin{equation}\label{eq:LearningProb2}
  \estim{\Dict} 
    \in \argmin_{\Dict} \PDist\bigl( \estim{\PDens}_{\TData}, \ModelDistr_{\Dict} \bigr)
  \quad\text{where}\quad
  \estim{\PDens}_{\TData} 
    := \frac{1}{N}\sum_{i=1}^{N} \delta_{\signal_i}.
\end{equation}
Here $\delta_{\signal_i}$ is the probability measure on $\RecSpace$ that has a unit point mass at $\signal_i \in \RecSpace$.

The next step is to specify $\PDist \colon \PClass{\RecSpace} \times \PClass{\RecSpace} \to \Real$ in \cref{eq:LearningProb} (and \cref{eq:LearningProb2}).
The \ac{KL} divergence is a common choice, so \cref{eq:LearningProb2} reads as 
\begin{equation}\label{eq:LearningProbKL}
  \estim{\Dict} 
    \in \argmin_{\Dict} \KL\bigl( 
      \estim{\PDens}_{\TData}
      \mid 
      \ModelDistr_{\Dict}  
    \bigr).
\end{equation}
If $\pdf_{\Dict} \colon \RecSpace \to \Real_+$ is the density for $\ModelDistr_{\Dict}$, then the \ac{KL} divergence is expressible as  
\begin{align*} 
  \KL\bigl( 
      \estim{\PDens}_{\TData} 
      \mid 
      \ModelDistr_{\Dict} 
    \bigr)
    &:= \frac{1}{N}\sum_{i=1}^N \Bigl\{ \log\Bigl(\frac{1}{N}\Bigr) - \log \pdf_{\Dict}(\signal_i) \Bigr\}
\\    
    &= \log\Bigl(\frac{1}{N}\Bigr) - \frac{1}{N}\sum_{i=1}^{N} \log \pdf_{\Dict}(\signal_i).
\end{align*}
Hence, the dictionary learning problem in \cref{eq:LearningProbKL} can be re-phrased as evidence maximisation:
\begin{equation}\label{eq:LearningProbKL2} 
   \estim{\Dict} \in \argmax_{\Dict} 
   \sum_{i=1}^{N} \log \pdf_{\Dict}(x_i).
\end{equation}

\subsection{Sparse dictionary representations}

To proceed, we introduce assumptions that are typical for sparse dictionary representations \cite{lewicki2000learning,lewicki1999probabilistic,girolami2001variational,tolooshams2020deep}. 

\begin{assumption}[Sparse dictionary representation] 
Consider the generative model in \cref{eq:GenModel} where:
\begin{enumerateA}
\item \label{as:a1}
  $\RecSpace = \Real^\signaldim$, $\DictCoeffSet = \Real^\dictcoeffdim$ with $\signaldim \ll \dictcoeffdim$. 
\item \label{as:a2} 
  Synthesis operator: $\SynthesisOp_{\Dict} \colon \Real^\dictcoeffdim \to \Real^\signaldim$ is linear, i.e.  
  \begin{equation}\label{eq:LinearSynthesisOp} 
    \SynthesisOp_{\Dict}(\dictcoeff) := 
    \DictMat \dictcoeff
    \quad\text{where $\DictMat \in \Real^{\signaldim \times \dictcoeffdim}$ }.
  \end{equation}
\item \label{as:a3} 
    Dictionary atoms have a fixed norm:
    \begin{equation}
        \|\dictatom\| = 1 \quad \text{for all } \dictatom \in \Dict.
    \end{equation}
\item \label{as:a4} 
  Dictionary coefficients are Laplace distributed, i.e., $\stdictcoeff \sim \PriorDictCoeff$ has corresponding density
  \begin{equation}\label{eq:LaplaceDensity}
    \pdf_{\stdictcoeff}(\dictcoeff) 
    = \frac{1}{(2 \dictcoeffpar)^\dictcoeffdim} 
    \exp\Bigl(-\frac{\Vert \dictcoeff \Vert_1}{\dictcoeffpar} \Bigr).
  \end{equation}
\item \label{as:a5} 
  $\stsignalnoise$ is independent of $\stdictcoeff$ and has Gaussian distribution with density 
  \begin{equation}\label{eq:GaussianDensity}
      \pdf_{\stsignalnoise}(\signal) = \frac{1}{(\sqrt{2\pi} \signalnoisepar)^n} 
      \exp\Bigl( -\frac{1}{2\signalnoisepar^2}\Vert \signal \Vert_2^2 \Bigr).
  \end{equation}
\end{enumerateA}
\end{assumption}
Assumption \ref{as:a1} merely states that the signal and latent spaces are both finite dimensional and the dictionary is over-complete, i.e., there are more dictionary atoms ($=\dictcoeffdim$) than the dimension of the signal ($=\signaldim$). Linearity assumption in \ref{as:a2} implies that synthesis operator $\SynthesisOp_{\Dict}$ can be represented by a matrix. However, it does not necessarily mean that each dictionary atom in $\Dict$ is represented by a column vector of $\DictMat$. As an example, we will use a convolutional synthesis operator where dictionary atoms (in $\Dict$) represent convolutional kernels and for each kernel there are columns in $\DictMat$ that correspond to different shifts of that kernel.
In addition, we fix the norm of dictionary atoms in assumption \ref{as:a3}. Without this assumption, one would have to take into account that distribution of the dictionary coefficients depends on the norm of corresponding dictionaries. This would make the model, in particular the assumption \ref{as:a4}, more complicated.  
The assumption \ref{as:a4} ensures that maximizing the posterior density $\dictcoeff \mapsto \pdf_{\stdictcoeff \mid \stsignal}(\dictcoeff \mid \signal)$ 
yields sparsity of dictionary coefficients.
Finally, in \ref{as:a5} we use Gaussian distribution to account for small inaccuracies in the generative model $\SynthesisOp_{\Dict}(\stdictcoeff)$. 

We proceed by making use of the above assumptions for computing the objective (log-evidence) in \cref{eq:LearningProbKL2}.
These assumptions yield an expression for the joint density for $(\stsignal,\stdictcoeff)$, so the idea is to express the desired density by marginalizing over dictionary coefficients $\stdictcoeff$:
\begin{equation}
\label{eq:marginal}
\begin{split}
    \sum_{i=1}^{N} \log \pdf_{\Dict}(x_i) 
    &= \sum_{i=1}^{N} \log \int_{\DictCoeffSet} \pdf_{\stsignal,\stdictcoeff}(\signal_i,\dictcoeff_i)\dint \dictcoeff_i 
    \\
    &= \sum_{i=1}^{N} \log \int_{\DictCoeffSet} 
      \pdf_{\stsignal \mid \stdictcoeff}(\signal_i \mid \dictcoeff_i)  \pdf_{\stdictcoeff}(\dictcoeff_i)  
    \dint \dictcoeff_i \\
    &= \sum_{i=1}^{N} \log \int_{\DictCoeffSet}
    \pdf_{\stsignalnoise}(\signal_i-\SynthesisOp_{\Dict}(\dictcoeff_i))  \pdf_{\stdictcoeff}(\dictcoeff_i)  
    \dint \dictcoeff_i.
\end{split}
\end{equation}
Note that all density functions in the above expression except for $\pdf_{\stdictcoeff}$ depend on dictionary $\Dict$, however we do not show this dependency to simplify the notation. Nevertheless, it is important to keep in mind that distributions $\pdf_{\stsignal,\stdictcoeff}$ and $\pdf_{\stsignal \mid \stdictcoeff}$ are not the true distributions, but the ones assumed by the model (since we are expressing the log evidence $ \log \pdf_{\Dict}(x_i)$).

The $\DictCoeffSet$-integrals above are computationally demanding.   
A standard machine learning approach to avoid this problem by decomposing the log evidence into two terms:
\begin{equation}\label{eq:ELBOsplit}
\begin{split}
    \log \pdf_{\Dict}(\signal) &= \KL(\ApproxPdf \mid \pdf_{\stdictcoeff \mid \stsignal})
    + \ELBO(\ApproxPdf)(\signal) \geq \ELBO(\ApproxPdf)(\signal)
\end{split}
\end{equation}
Note that $\ApproxPdf \colon \DictCoeffSet \to [0,1]$ above is \emph{any} probability density function on $\DictCoeffSet$ and $\ELBO(\ApproxPdf) \colon \RecSpace \to \Real$ is the  \ac{ELBO} that is defined as 
\begin{equation}\label{eq:ELBODef}
\begin{split}
    \ELBO(\ApproxPdf)(\signal) &= 
    \int_{\DictCoeffSet} \ApproxPdf(\dictcoeff) \log  \frac{\pdf_{\stsignal,\stdictcoeff}(\signal,\dictcoeff)}{\ApproxPdf(\dictcoeff)}  \dint \dictcoeff \\
    &   = \ExpectApproxPdf \bigl[ 
        \log\pdf_{\stsignal,\stdictcoeff}(\signal,\ApproxStCoef) \bigr]   
        - \ExpectApproxPdf\bigl[ \log  \ApproxPdf(\ApproxStCoef) \bigr],
\end{split}
\end{equation}
where $\ApproxStCoef$ is any $\DictCoeffSet$-valued random variable with a probability density function $\ApproxPdf$
and $\ExpectApproxPdf$ denotes the expectation w.r.t. this density.
Next, the inequality in  \cref{eq:ELBOsplit} holds for any $\ApproxPdf$ since the \ac{KL} divergence is always positive. 
Hence, instead of directly maximizing the log evidence $\log \pdf_{\Dict}(\signal)$ with respect to $\Dict$ as in \cref{eq:LearningProbKL2}, one can maximize the \ac{ELBO}. 
Furthermore, if $\ApproxPdf$ approximates $\pdf_{\stdictcoeff \mid \stsignal}$, the \ac{KL} divergence is small and therefore the lower bound is tight, see \cite{bishop2006pattern} for further details.

We consider $\ApproxPdf$ being a Laplace distribution concentrated around the mode of $\pdf_{\stdictcoeff \mid \stsignal}$. This choice is supported by the following intuition: The mode of the posterior $\dictcoeffmode$ is a sparse vector. Therefore, most elements of its elements are 0. If the representation is accurate (the variance of $\stsignalnoise$ is small), then the dictionary model is capable of explaining a large part of the variability of $\stsignal$ and $\pdf_{\stsignal \mid \stdictcoeff }(\signal \mid \dictcoeff) $ is larger than $ \pdf_{\stsignal}(\signal)$ close to the mode $\dictcoeffmode$. If the representation is unique, the opposite holds further away from the mode. Therefore,  along most of the dimensions the posterior
\begin{equation}
    \pdf_{\stdictcoeff \mid \stsignal} (\dictcoeff \mid \signal) = \frac{\pdf_{\stsignal \mid \stdictcoeff }(\signal \mid \dictcoeff)}{\pdf_{\stsignal}(\signal) } \pdf_{\stdictcoeff}(\dictcoeff)
\end{equation}
is even spikier than the prior $\pdf_{\stdictcoeff}(\dictcoeff)$. 

Furthermore, the following theorem shows that a slight relaxation of the lower bound will, given this approximation of the posterior $\pdf_{\stdictcoeff \mid \stsignal}$ by $\ApproxPdf$, result in the optimization problem \cref{eq:typical}.
Typical approaches for dictionary learning are based on solving this optimisation, so the theorem below offers a statistical interpretation of the dictionary learning procedure.
The theorem shows in particular that if the model is sparse, then solving \cref{eq:typical} essentially amounts to maximizing the \ac{ELBO} for a Laplace distributed posterior. Moreover, if this posterior is appropriate (\ac{KL} divergence with the true posterior is small), we are maximizing log-likelihood of image samples and hence minimizing the \ac{KL} divergence between empirical distribution of images and distribution generated by our model.
\begin{theorem}
Consider the generative model in \cref{eq:GenModel} and assume  \ref{as:a1}-\ref{as:a5} holds. 
Also, let $\ApproxPdf$ denote the density of a Laplace distributed $\DictCoeffSet$-valued random variable centered at the mode for the posterior, i.e.
\begin{equation}
    \ApproxPdf(\dictcoeff \mid \signal) = \frac{1}{(2\dictcoeffmodepar)^\dictcoeffdim} \exp\left( -\frac{\|\dictcoeff-\dictcoeffmode \|_1}{\dictcoeffmodepar} \right)
    \quad\text{for fixed  $\dictcoeffmodepar>0$,}
\end{equation}
with 
\begin{equation}
    \dictcoeffmode
    := \argmax_{\dictcoeff \in \DictCoeffSet}\, \log \pdf_{\stdictcoeff \mid \stsignal}(\signal,\dictcoeff) 
    = \argmax_{\dictcoeff \in \DictCoeffSet}\, \log \pdf_{\stsignal, \stdictcoeff}(\signal,\dictcoeff).
\end{equation}
Then 
\begin{equation}\label{eq:elbo_lb}
\begin{split}
    \ELBO(\ApproxPdf)(\signal) &= - \ExpectApproxPdf \bigl[ f(\signal,\ApproxStCoef) \bigr] + C(\signalnoisepar, \dictcoeffpar, \dictcoeffmodepar) \\
    &\geq  - f(\signal, \dictcoeffmode) - \dictcoeffdim\frac{(\dictcoeffmodepar)^2}{\signalnoisepar^2} - \dictcoeffdim\frac{\dictcoeffmodepar}{\dictcoeffpar}  +  C(\signalnoisepar, \dictcoeffpar, \dictcoeffmodepar)
\end{split}
\end{equation}
where   
\begin{equation}
\begin{split}
f(\signal,\dictcoeff) &:= \dfrac{1}{2\signalnoisepar^2}\|\DictMat \dictcoeff-\signal\|^2 + \dfrac{1}{\dictcoeffpar}\|\dictcoeff\|_1 
\\
C(\signalnoisepar, \dictcoeffpar, \dictcoeffmodepar) &:= - \frac{\signaldim}{2}\log 2\pi \signalnoisepar^2 + \dictcoeffdim \log \frac{\dictcoeffmodepar}{\dictcoeffpar}  + 1. 
\end{split}
\end{equation}
Finally, the gap between $\ELBO(\ApproxPdf)$ and its lower bound in \cref{eq:elbo_lb} is bounded by $\dfrac{\dictcoeffmodepar}{\dictcoeffpar}\|\dictcoeffmode\|_0$.
\end{theorem}
\begin{proof}
First,  we express $\ELBO(\ApproxPdf)$ as defined in \cref{eq:ELBODef} by expressing its parts
\begin{equation}
\label{eq:ELBOpart1}
\begin{split}
    \ExpectApproxPdf\bigl[ 
  \log  \ApproxPdf(\ApproxStCoef \mid \signal) \bigr] 
  &= 
     - \dictcoeffdim \log 2 \dictcoeffmodepar - \frac{1}{ \dictcoeffmodepar} \ExpectApproxPdf \bigl[ 
       \|\ApproxStCoef - \dictcoeffmode\|_1  
     \bigr] \\
  &= - \dictcoeffdim \log 2\dictcoeffmodepar - 1
\end{split}
\end{equation}
and 
\begin{equation}
\label{eq:ELBOpart2}
\begin{split}
    \ExpectApproxPdf
    \bigl[ \log \pdf_{\stsignal,\stdictcoeff}(\signal,\ApproxStCoef) 
    \bigr]
    &= \ExpectApproxPdf\Bigl[ 
      \log \bigl( \pdf_{\stsignal \mid \stdictcoeff}(\signal \mid \ApproxStCoef) \pdf_{\stdictcoeff}(\ApproxStCoef) \bigr)
    \Bigr]\\
    &= - \frac{\signaldim}{2}\log 2\pi \signalnoisepar^2 - \dictcoeffdim\log2\dictcoeffpar - \ExpectApproxPdf\bigl[ f(\signal,\ApproxStCoef) \bigr]
\end{split}
\end{equation}
Set $\Delta \stdictcoeff := \ApproxStCoef - \dictcoeffmode$, then
\begin{equation}
\begin{split}
    f(\signal,\ApproxStCoef) 
    &= \frac{1}{2\signalnoisepar^2}\bigl\| \DictMat(\dictcoeffmode + \Delta \stdictcoeff ) - \signal \bigr\|^2 
    + \frac{1}{\dictcoeffpar}\|\dictcoeffmode + \Delta \stdictcoeff \|_1 \\
    &= \frac{1}{2\signalnoisepar^2}\|\DictMat \dictcoeffmode-\signal\|^2
    + \frac{1}{\signalnoisepar^2}(
        \DictMat \dictcoeffmode-\signal)^{\top}
        \DictMat \Delta \stdictcoeff \\
    &+ \frac{1}{2\signalnoisepar^2}\|\DictMat \Delta \stdictcoeff \|^2
    + \frac{1}{b}\|\dictcoeffmode + \Delta \stdictcoeff \|_1 
\end{split}
\end{equation}
Now, we make the following observations: 
\begin{itemize}
\item $\ExpectApproxPdf [\Delta \stdictcoeff] = \ExpectApproxPdf [\ApproxStCoef] - \dictcoeffmode = 0$,
\item $\ExpectApproxPdf\bigl[ \|\DictMat \Delta \stdictcoeff \|^2 \bigr] = \ExpectApproxPdf\bigl[ \Delta \stdictcoeff^{\top} \DictMat^{\top} \DictMat \Delta \stdictcoeff  \bigr] = 2(\dictcoeffmodepar)^2\tr(\DictMat^{\top} \DictMat) = 2(\dictcoeffmodepar)^2\dictcoeffdim$, 
\item $\|\ApproxStCoef\|_1$ is folded Laplace distribution and its expected value has been derived analytically in \cite{Liu2015AFL}:
    \begin{equation}
        \ExpectApproxPdf\bigl[ \|\ApproxStCoef\|_1 \bigr] = \|\dictcoeffmode\|_1 + \dictcoeffmodepar\sum_{i=1}^\dictcoeffdim \exp\left( \frac{-|\dictcoeffmode_i|}{\dictcoeffmodepar} \right).
    \end{equation}
\end{itemize}  
Then, we can conclude that
\begin{equation}
    \label{eq:Ef_lb}
\begin{split}
    \ExpectApproxPdf \bigl[ f(\signal,\ApproxStCoef) \bigr] &= f(\signal,\dictcoeffmode) + \dictcoeffdim\frac{(\dictcoeffmodepar)^2}{\signalnoisepar^2} + \frac{\dictcoeffmodepar}{\dictcoeffpar} \sum_{i=1}^\dictcoeffdim \exp\left( \frac{-|\dictcoeffmode_i|}{\dictcoeffmodepar} \right)\\
    &\leq f(\signal,\dictcoeffmode) + \dictcoeffdim \frac{(\dictcoeffmodepar)^2}{\signalnoisepar^2} +\dictcoeffdim \frac{\dictcoeffmodepar}{\dictcoeffpar}
\end{split}
\end{equation}
The desired lower bound on $\ELBO(\ApproxPdf)$ given in \cref{eq:elbo_lb} now follows from combining \cref{eq:Ef_lb} with \cref{eq:ELBOpart1,eq:ELBOpart2}.
\end{proof}
We conclude with some remarks that relate to consequences of the theorem.

Note first that the gap $\frac{\dictcoeffmodepar}{\dictcoeffpar}\|\dictcoeffmode\|_0$ is relatively small. It constitutes a $\|\dictcoeffmode\|_0/m$ fraction of the term $m \dictcoeffmodepar/\dictcoeffpar$. 
Thus, with sufficient sparsity we are very close to optimizing \ac{ELBO} even without sampling from the posterior.

Next, a similar result can be shown assuming that $\ApproxPdf$ is a Gaussian density function, however in this case, the gap between \ac{ELBO} and our optimization objective (lower bound of \ac{ELBO}) might be larger. In any case, whether assumption on the posterior shape is appropriate or not is quantified by \ac{KL} term. 

Finally, this theorem also points to a connection between dictionary learning and \acp{VA}. The synthesis operator in dictionary learning can be seen as a decoder with one convolutional layer, the sparse coding defines an encoder. However, in dictionary learning the posterior is parametrized by assuming it is Laplace distributed around the mode, while \acp{VA} assume Gaussian distribution. Since Laplace distribution is very spiky a good (even though biased) approximation of \ac{ELBO} is obtained without sampling from the posterior. Thus, the encoder does not have to predict variance of the posterior distribution, it is sufficient to predict the mode. In both cases the assumption that posterior has a certain form might not be appropriate, leaving the opportunity for both methods to fail.

\subsection{Implementation of dictionary learning}
\label{ssec:dl_implementation}
 
 In this section we discuss practical implementation of the synthesis operator  $\SynthesisOp_{\Dict} \colon \DictCoeffSet \to \RecSpace$ and the procedure of learning a good dictionary $\Dict$.

\subsubsection{Generative model}\label{sec:ImplGenMod}
In the \cref{sec:dl_theory} we made two assumptions regarding the generative model in \cref{eq:GenModel}. First, the synthesis operator \ref{as:a2} is linear, and second, dictionaries have fixed norm \ref{as:a3}. In this section we present further choices that are made to apply the model in practice.

As commonly done, we use the generative model \cref{eq:GenModel} to model only high frequency features, while a low frequency component is estimated using \ac{FBP} with a low-pass filter. Secondly, we use non-overlapping image patches to learn dictionary atoms. Further, follows the discussion of these choices.

\paragraph{Subtracting the low frequency component} 
The dictionary model \cref{eq:GenModel} needs to have sufficient capacity to represent ``natural'' signals in $\RecSpace$. 
Digitizing such signals will in many applications, like imaging, result in high dimensional arrays. 
Sparsely representing the entire signal with a fixed dictionary will therefore require impractically many atoms. In addition, learning those atoms will be challenging, bearing in mind the limited number of training data examples that are available in practice. A common approach to address the above is to use dictionaries for representing only the high frequency component of the signal. 

Using a dictionary of the above type in signal reconstruction must include a step that recovers the low frequency component of the signal from noisy data in order to remove it.  
For image de-noising, one can simply recover the low frequency component by low-pass filtering data (noisy image).
The corresponding approach in \ac{CT} image reconstruction is to use \ac{FBP} with a cut-off frequency for the reconstruction kernel that is far below the Nyquist frequency that is dictated by sampling theory.
This ensures an over-smoothed image representing the low frequency component. 
In our experiments, we set the cut-off frequency to 10\% of the Nyquist frequency.

The above approach also requires one to train the dictionary against high frequency components of natural signals, i.e., one needs to remove low frequency components of images in the training set. This could be done by simply applying a low-pass filter to the images. However, in \ac{CT} image reconstruction, we don't have images available and the low frequency component must be obtained differently. Using different procedures for extracting low frequency components during dictionary learning and during \ac{CT} reconstruction is likely lead to sub-optimal results. For this reason, we employ the following procedure to remove low frequency components of images in the training set: First, generate noise free synthetic data by applying the \ac{CT} forward operator on the high resolution CT images in the training set. Next, compute the corresponding low frequency component of the images by applying \ac{FBP} on this noise free synthetic data.
This results in an approach for dictionary learning adapted for \ac{CT} reconstruction. 
Finally, the approach is unsupervised in the sense that it only requires access to high quality \ac{CT} images, i.e., it does not require any access to \ac{CT} projection data.

\paragraph{Training on patches}\label{patchTraining} 
When we learn dictionary $\Dict$ we use a patch-based synthesis operator $\SynthesisOp^p_{\Dict} \colon \DictCoeffSet  \to \RecSpace$ that  generates an image from sparse representations of on non-overlapping image patches. In practice, we implement it as a convolutional operator with a stride that is equal to the size of a dictionary atom. However, when we use the learned dictionary for regularization, we set the stride to 1, i.e. use a regular convolutional synthesis operator $\SynthesisOp_{\Dict} \colon \DictCoeffSet  \to \RecSpace$.
We motivate this choice in a discussion below.

The generative model defined in \cref{eq:SynthesisOp} traditionally has been applied to image patches instead of  whole images. To avoid ``block'' artifacts in signal restoration tasks (like de-noising and reconstruction), the whole image has been processed by applying the model to overlapping patches independently and then averaging the results. 

Patch-based dictionary learning described above and convolutional dictionary learning define different ways to learn a dictionary. However, when a learned dictionary is applied for regularization in some task, the two approaches rely on very similar generative models. In a patch-based setting for each patch there is a corresponding sparse signal representation (dictionary coefficients). If these representations are combined together to form  coefficient maps $\dictcoeff_i \in \RecSpace$ for each dictionary element $\dictatom_i \in \Dict$, then a full image can be generated by applying a convolutional synthesis operator as defined \cref{eq:SynthesisOpConv}. In fact, the only difference between the two methods is that in the patch-based approach dictionary coefficients are computed independently for each patch. In contrast, in the convolutional framework the sparse coding is done jointly for all overlapping patches in the image. 

Convolutional generative model have several advantages. Joint optimization with respect to coefficients allows to obtain closer approximation of an image given the same sparsity level, making the model more flexible. Moreover, a dictionary does not have to contain shifted versions of the same atoms, which makes the model more efficient \cite{chambolle2020convex}. On the other hand, in \cite{simon2019rethinking} authors argue that atoms in a convolutional dictionary suitable for representing natural images have high correlation with shifted versions of themselves, which makes uniqueness and stability guarantees for the sparse coding step obsolete. Non-uniqueness of the sparse representation is problematic in training/dictionary learning, because in this case the posterior distribution of coefficients is not unimodal and the Laplace approximation is not appropriate. Subsequently, solving dictionary learning problem \cref{eq:typical} might not maximize the log-likelihood \cref{eq:ELBOsplit}. This is supported by the observation that using convolutional dictionaries for regularization in de-noisining has been problematic in practice. Indeed, works that we know have succeeded to learn only small ($8\time8$) dictionary atoms for representing high frequency components. 

Our initial goal was to use convolutional dictionary learning with large dictionary atoms, to provide stronger regularization in solving inverse problems. However, we also found that learning such atoms in convolutional setting is very hard, since it is hard to ensure both sparsity and that dictionary atoms are regularly updated. First, we tried to address this issue by dropout. However, we discovered that best option in terms of results and computational efficiency is to learn the dictionary on patches, by using a patch-bases synthesis operator $\SynthesisOp^p_{\Dict}$.

\paragraph{Regularization with convolutional synthesis operator}  
Since we discovered that it is beneficial to learn a dictionary on non-overlapping patches, it would be natural to apply this dictionary for regularization as it is usually done in the patch-based approach. However, we find that using the convolutional model with the learned dictionary gives slightly better results in practice, see \cref{app:CompPatch} for comparison. Therefore, we substitute operator $\SynthesisOp^p_{\Dict}$ by $\SynthesisOp_{\Dict}$.
As a consequence, a different generative model is used in training and in application. A natural question is to elucidate how this affects the  distribution $\ModelDistr_{\Dict}$ in \cref{sec:EvidenceMax} that is given by the learned dictionary. 

First, the new synthesis operator will increase the model capacity, therefore $\signal \mapsto f(\signal, \dictcoeffmode)$ in \cref{eq:elbo_lb} will decrease. Unfortunately, the model will also become better at representing unwanted features, which is reflected by an increase in the other terms in \cref{eq:elbo_lb}, like the dimension $m$ of the coefficient space.
This increase in model capacity most likely leads to overall decrease \ac{ELBO}. However, the increase in model capacity is limited by the fact that during the training on non-overlapping patches, dictionaries are forced to learn shifted versions of themselves. Therefore, the posterior distribution of the coefficients becomes multi-modal and $\KL(\ApproxPdf \mid \pdf_{\stdictcoeff \mid \stsignal})$ increases. This explains how the approach can still lead to good results in practice.

\subsubsection{Solving the non-convex optimization problem}
In principle one could learn the best generative model by maximizing the right-hand-side of \cref{eq:elbo_lb} jointly with respect to $\Dict$ and the hyper-parameters $\signalnoisepar$, $\dictcoeffpar$, and $\dictcoeffmodepar$.
However, this results in a non-convex optimization problem whose solution depends on the initialization.
More precisely, if one initializes with poor dictionaries and optimize with respect to hyper-parameters, one risks converging to a trivial solution where all image features are explained by noise and not the dictionary. 

For the above reasons, we fix hyper-parameters and for given training data  $\{\signal_i\}, i = 1,\dots N$ (with low frequency components removed as outlined in \cref{sec:ImplGenMod}), we learn the dictionary by solving 
\begin{equation}
\label{eq:DL_inpractice}
\begin{split}
    &\min_{\Dict}  \sum_{i=1}^N  \min_{\dictcoeff_i}  \|\SynthesisOp^p_{\Dict}(\dictcoeff_i) - \signal_i\|^2+ \regparamrec \|\dictcoeff_i\|_1,\\
    &\|\dictatom\| = 1, \quad \text{for all } \dictatom \in \Dict.
\end{split}
\end{equation}

\paragraph{Stochastic optimization with alternating steps}

Since the number of available samples $N$ is high, we solve \cref{eq:DL_inpractice} using stochastic optimization with batch size 1. 
This means that at each iterate $k$, only one random image $\signal_{i_k}$ is used to update the dictionary. 
In addition, we randomly crop a smaller image patch from the image. 

We solve the optimization problem by minimizing the objective with respect to dictionary and coefficients in alternating steps.
First, dictionary coefficients $\dictcoeff_k$ are computed by solving a convex sub-problem:
\begin{equation}\label{eq:DictCoeffUpdate}
   \dictcoeff_{k}  = \argmin_{\dictcoeff}  
          \bigl\Vert\SynthesisOp_{\Dict_k}(\dictcoeff) - \signal_{i_k} \bigr\Vert_2^2
          + \regparamrec_k \Vert \dictcoeff \Vert_1.
\end{equation}
This is done by applying \ac{FISTA} \cite{Beck:2009aa} for a fixed number of iterates. 
Secondly, the dictionary is updated using the previous iterate $\Dict_k$ and the $\Dict$-gradient of the objective in \cref{eq:DictCoeffUpdate} at $\Dict_k$:
\begin{equation}\label{eq:GenericUpdate}
    \Dict_{k+1} = \operatorname{OptUpdate}\biggl( \Dict_{k}, 
    \nabla_{\!\Dict} \Vert \SynthesisOp_{\Dict_k}(\dictcoeff_{k}) - \signal_{i_k} \Vert_2^2
     \biggr).
\end{equation}
In particular, simple \ac{SGD} has updates of the form
\begin{equation}\label{eq:GradDescentUpdate}
       \Dict_{k+1} 
    =  \Dict_{k} - \alpha_k 
     \nabla_{\!\Dict} \Vert \SynthesisOp_{\Dict_k}(\dictcoeff_{k}) - \signal_{i_k} \Vert_2^2.
\end{equation}
We find that using the Adam optimizer \cite{Kingma:2014aa} instead of \ac{SGD} leads to better results. In Adam the accumulated gradient is normalized with respect to accumulated variance of the gradients. 
In dictionary learning, this helps to avoid so called ``dead'' atoms - the atoms that are not picked up in the sparse coding step and subsequently not updated. Indeed, we learn two times more different dictionary atoms while using the Adam optimizer compared to regular \ac{SGD}.

Note that $\nabla_{\!\Dict}$ in \cref{eq:GenericUpdate,eq:GradDescentUpdate} is the $\Dict$-gradient  of 
$\Dict \mapsto \bigl\Vert
   \SynthesisOp_{\Dict}(\dictcoeff) - \signal 
 \bigr\Vert_2^2$ 
which can be explicitly expressed as 
\[
\nabla_{\!\Dict}
  \bigl\Vert
        \SynthesisOp_{\Dict}(\dictcoeff) - \signal 
      \bigr\Vert_2^2 
=
2 \bigl[ \partial_{\Dict} \SynthesisOp_{\Dict}(\dictcoeff)\bigr]^*
    \bigl( 
      \SynthesisOp_{\Dict}(\dictcoeff) - \signal
    \bigr).
\]
Here, $\partial_{\Dict} \SynthesisOp_{\Dict}(\dictcoeff)$ is the $\Dict$-Jacobian of $\Dict \mapsto \SynthesisOp_{\Dict}(\dictcoeff)$.

\paragraph{Adaptive choice of regularization parameter}
If the regularization parameter $\regparamrec$ is constant during the training, many dictionary atoms will not be updated in the beginning of the optimization process and therefore will not be used later. Thus, it is important to be flexible with sparsity level $\regparamrec$ to insure that majority of the atoms get a good start. We address this problem by adopting $\regparamrec$ during the optimization so that the average number of non-zero elements in $\dictcoeff$ stays approximately constant. This results in $\regparamrec$ rising in the beginning of the optimization and then stabilizing around one value. 

 This is done by estimating the average sparsity level on the validation set $\estim{s} = \frac{1}{N_{val}} \sum_{i=1}^{N_{val}}\Vert \dictcoeff^{t}_i \Vert_0$. 
If $s$ is the pre-set sparsity level, then during every validation step we have for some small constant $c$ 
\begin{equation}
\begin{split}
  \regparamrec_{t+1}
    &=   \begin{cases}
          \regparamrec_t + c (\estim{s} - s  ) 
                 & \text{ if $\vert  \estim{s} - s\vert > 0.2 s$, and $t \operatorname{mod} 10 = 0$} \\
          \regparamrec_t & \text{otherwise.}
        \end{cases} \\    
\end{split}
\end{equation}

An alternative approach to dictionary initialization is to apply clustering algorithms to image patches and initialize atoms as cluster centres \cite{agarwal2013exact, arora2014new}. However, this approach is likely to be  too demanding from the computational point of view, given a large number of image patches that we have.

\section{Dictionary based regularization in inverse problems}
\label{sec:dl_regularization}

An inverse problems is a problem of recovering an unknown signal $\signaltrue \in \RecSpace$ from data $\data \in \DataSpace$ that represents indirect  noisy observations of the signal.
In statistical terms we are interested in the distribution of the conditional random variable $( \stsignal \mid \stdata = \data )$, where 
\begin{equation}\label{datamodel}  
\stdata = \ForwardOp(\stsignal) + \stdatanoise. 
\end{equation}
Here, the mapping $\ForwardOp \colon \RecSpace \to \DataSpace$ (forward operator) is assumed to be known and  $\stdatanoise$ is a random variable that generates observation errors.  

Given a prior distribution $\stsignal \sim \SignalPrior$ the posterior density is expressible as 
\begin{equation}\label{posteq}
    \pdf_{\stsignal \mid \stdata}(\signal \mid \data) \propto \pdf_{\stdata \mid \stsignal}(\data \mid \signal ) \pdf_{\stsignal}(\signal) = \pdf_{\stdatanoise \mid \stsignal}\bigl(\data-\ForwardOp(\signal) \mid \signal \bigr) \pdf_{\stsignal}(\signal) 
\end{equation}
where the density $\pdf_{\stdatanoise \mid \stsignal}$ for the observation error is usually know. 
Furthermore, the prior $\SignalPrior$ can be approximated by $\ModelDistr_{\estim{\Dict}}$, where $\estim{\Dict}$ is a learned dictionary:
\begin{equation}\label{eq:PriorApprox}
\begin{split}
    \pdf_{\stsignal}(\signal)  
    &\approx  \int \pdf_{\stsignal \mid \stdictcoeff} (\signal \mid \dictcoeff) \pdf_{\stdictcoeff}(\dictcoeff) \dint \dictcoeff.
\end{split}
\end{equation}

The true solution $\signaltrue$ to the inverse problem \cref{datamodel} can be estimated by the \ac{MAP} estimator, which is defined as the maximum of $\signal \mapsto \pdf_{\stsignal \mid \stdata}(\signal \mid \data)$.
Assuming the true prior can be approximated by a learned  dictionary as in  \cref{eq:PriorApprox} introduced a marginalization over $\stdictcoeff$, which is computationally unfeasible, so an alternative is to compute unmarginalized MAP estimator:
\begin{equation}\label{eq:RecoProblemGeneric}
\begin{split}
    (\estim{\signal}, \estim{\dictcoeff}) 
    &\in \argmax_{(\signal, \dictcoeff) \in \RecSpace \times \DictCoeffSet} \pdf_{\stsignal, \stdictcoeff \mid \stdata}(\signal, \dictcoeff \mid \data)\\
    &= \argmax_{(\signal, \dictcoeff) \in \RecSpace \times \DictCoeffSet} \pdf_{\stdata \mid \stsignal} (\data \mid \signal) \pdf_{\stsignal \mid \stdictcoeff} (\signal \mid \dictcoeff) \pdf_{\stdictcoeff}(\dictcoeff) \\
    &= \argmax_{(\signal, \dictcoeff) \in \RecSpace \times \DictCoeffSet} \pdf_{\stdatanoise \mid \stsignal}\bigl(\data-\ForwardOp(\signal) \mid \signal \bigr)  \pdf_{\stsignalnoise} (\signal - \SynthesisOp_{\estim{\Dict}}(\dictcoeff)) \pdf_{\stdictcoeff}(\dictcoeff) \dint \dictcoeff. 
\end{split}
\end{equation}

Inserting the specific expressions for the densities of the random variables $(\stdatanoise \mid \stsignal)$, $\stsignalnoise$, and $\stdictcoeff$ into \cref{eq:RecoProblemGeneric} yields
\begin{equation}\label{eq:RecoProblem}
\begin{split}
    (\estim{\signal}, \estim{\dictcoeff}) 
    &\in\argmax_{(\signal, \dictcoeff)\in \RecSpace \times \DictCoeffSet} \DataLogLikelihood\bigl(\ForwardOp (\signal),  \data\bigr) + \regparamrec_1 \bigl\| \signal -\SynthesisOp_{\estim{\Dict}}(\dictcoeff) \bigr\|_2^2 +  \regparamrec_2 \|\dictcoeff\|_1
\end{split}
\end{equation}
where $\DataLogLikelihood\bigl(\ForwardOp (\signal),  \data \bigr) := - \log \pdf_{\stdatanoise \mid \stsignal}\bigl(\data-\ForwardOp(\signal) \mid \signal \bigr) $ is the data (negative) log-likelihood and the regularization parameters are $\regparamrec_1 := 1/\signalnoisepar$ and  $\regparamrec_2 := 1/\dictcoeffpar$.
An equivalent bi-level optimization formulation to \cref{eq:RecoProblem} reads as 
\begin{numcases}{}
   \estim{\signal} \in \argmin_{\signal \in \RecSpace} 
  \DataLogLikelihood\bigl(\ForwardOp(\signal),\data\bigr)
   +
   \regparamrec_1 \RegFunc_{\estim{\Dict}}(\signal)  
   & \label{eq:RecoDict1}
   \\
   \RegFunc_{\Dict}(\signal) \in \min_{\dictcoeff \in \DictCoeffSet}
     \bigl\Vert 
       \signal - \SynthesisOp_{\Dict}(\dictcoeff)
     \bigr\Vert_2^2 
     + \frac{\regparamrec_2}{\regparamrec_1} \Vert \dictcoeff \Vert_1. 
   & \label{eq:VarProb1}
\end{numcases}

\subsection{Implementation}
Signals (images) and measured data are in practice digitized and  represented by arrays, i.e., $\RecSpace = \mathbb{R}^\signaldim$ and $\DataSpace = \mathbb{R}^\datadim$ in \cref{datamodel}.
Next, we assume that observation noise $(\stdatanoise \mid \stsignal = \signal) $ has a Poisson distribution conditioned on noise-free data $\ForwardOp(\signaltrue)$ generated by the signal. 
To simplify computations we use a quadratic approximation for the data log-likelihood in \cref{eq:RecoProblem}, that results in the following weighted $l_2-$norm \cite{elbakri2002statistical}:
\begin{equation}\label{eq:WeightedL2}
    \DataLogLikelihood\bigl(\ForwardOp (\signal), \data\bigr) = \sum_{i = 1}^{l} w_i \bigl\Vert \ForwardOp (\signal)_i - \data_i \bigr\Vert_2^2
    \quad\text{where weights $w_i = e^{-\data_i}$.} 
\end{equation}

There are now several ways to solve the corresponding convex optimisation problem in \cref{eq:RecoProblem}. First, \ac{FISTA} \cite{Beck:2009aa} can be formulated with $(\signal, \dictcoeff)$ as the control variable. 
A different approach is to use an alternating scheme and apply a gradient descent step to variable $\signal$ and a proximal gradient descent to $\dictcoeff$ in alternating steps. 
Such a scheme would guarantee a monotonic decrease of the objective function given an appropriate selection of the step size \cite{beck2017first}. 
In order to speed up the optimization, we use an accelerated gradient descent \cite{Nesterov:1983} to update $\signal$, and an accelerated proximal gradient descent (as in \ac{FISTA} \cite{Beck:2009aa}) to update $\dictcoeff$, which results in the following iterative scheme: 
\begin{equation}
\begin{split}
    t_{k+1} &= \frac{1+\sqrt{1+4t_k^2}}{2}\\
    \signal_{k+1} &= \signal^{\prime}_{k} - \frac{1}{L_{\signal}} \nabla_{\signal} f (\signal^{\prime}_{k}, \dictcoeff^{\prime}_{k}) \\
    \signal^{\prime}_{k+1} &= \signal_{k+1} + \frac{t_k-1}{t_{k+1}} \left (\signal_{k+1} - \signal_{k}\right ) \\
    \dictcoeff_{k+1} 
    &= \operatorname{prox}_{g}^{ 1/L_{\dictcoeff}}\Bigl(
      \dictcoeff^{\prime}_{k+1} - \frac{1}{L_{\dictcoeff}} \nabla_{\dictcoeff} f(\signal_{k+1}, \dictcoeff^{\prime}_{k})
      \Bigr)
    \\
    \dictcoeff^{\prime}_{k+1} &= \dictcoeff_{k+1} + \frac{t_k-1}{t_{k+1}}\left(\dictcoeff_{k+1} - \dictcoeff_{k}\right)\\
\end{split}        
\end{equation}
for $k\geq 1$ where 
$
    f(\signal, \dictcoeff) := \DataLogLikelihood\bigl(\ForwardOp (\signal),  \data\bigr) + \regparamrec_1 \bigl\| \signal -\SynthesisOp_{\estim{\Dict}}(\dictcoeff) \bigr\|_2^2 
$
is the smooth part of the objective in \cref{eq:RecoProblem} and $L_{\signal}, L_{\dictcoeff}$ are the Lipschitz constants of the gradient of $f$ with respect to $\signal$ and $\dictcoeff$. The proximal operator for the non-smooth part of the objective $g(\dictcoeff) =\regparamrec_2 \|\dictcoeff\|_1$ admits a closed form expression
\begin{equation*}
\begin{split}
    \operatorname{prox}_{g}^{\gamma}(u) &=  \argmin g(z) + \frac{1}{2\gamma} \|z - u\|^2 \\
    &= \operatorname{sign}(u) 
       \max\bigl(|u| - \gamma\regparamrec_2, 0\bigr).
\end{split}
\end{equation*}
The essential difference between the above scheme and \ac{FISTA} (for the variable $(\signal, \dictcoeff)$) is that we are separately estimating the step size for $\signal$ and $\dictcoeff$, which results in using larger step sizes and, hence, faster convergence. 
However, we are not aware of any theoretical results guaranteeing  convergence of our accelerated scheme in the general case, we verify that it results in a monotonic decrease of the objective function values in our experiments. 
Note, that we perform the optimization for the limited number of iterations, thus we do not necessarily reach the state of full convergence. 
In fact, running the process until full convergence leads to sub-optimal results. 
This observation that early stopping provides additional regularisation has been seen with other variational methods, like \ac{TV} regularisation \cite{Effland:2020aa}.

\section{Application to \ac{CT} image reconstruction}
\label{sec:ct_rec}
\subsection{Experimental setting}

To perform the experiments we simulate projection data using images of human abdomen provided by the 2016 AAPM Low Dose CT Grand Challenge \cite{McCollough:2017aa}. The dataset contains \ac{CT} scans of 10 patients, 9 of which are used for training and one is reserved for testing. To avoid computational burden we work in a 2D setting, splitting 3D volumes into axial slices of size $512 \times 512$. This results in 2168 and 210 image-data pairs for training and testing respectively. Moreover, 1\% of the training images are  separated into validation set, which is used for setting the hyper-parameters for the evaluated methods. The choice of  hyper-parameters is made with an objective to minimize the mean squared error (MSE) on the validation set. Lastly, we compute evaluation metrics as for one 3D image, when we report the results, even though the methods are applied in a slice by slice manner.

Since the images are given in Hounsfield scale, we re-scale those by a factor $\mu_0/1000$, where $\mu_0 = 0.0192 \text{mm}^{-1}$ is X-ray attenuation of water at the mean X-ray energy of 70~keV. 
The noise-free tomographic data $y = \ForwardOp(x)$ is simulated using cone beam ray transform with 1000 detector elements and 1000 projection angles. According to the Beer–Lambert law the number of  photons absorbed by the detector is 
\begin{equation}
    N_{d} = N_0 e^{-\ForwardOp(\signal)}.
\end{equation}
Here $N_0$ is an average number of photons that would reach a detector element in an empty space. Since $N_0$ is proportional to radiation intensity, smaller values correspond to a lower dose.
For a single energy X-ray noise is modelled as a Poisson distributed random variable with mean $N_{d}$:
\begin{equation}
    N_{\text{noisy}} = \text{Poisson}\left( N_0 e^{-\ForwardOp(\signal) } \right)
\end{equation}
Finally, we linearize the noisy data
\begin{equation}
    \data_{\text{noisy}} = - \ln\left(\frac{1}{N_0} N_{\text{noisy}} \right)
\end{equation}
so that $\ForwardOp(x) \approx \data_{\text{noisy}}$. In our experiments we use $N_0 = 50\,000$, which corresponds to a noise level experienced in clinical low-dose \ac{CT} imaging.

We learn 512 dictionary atoms of size $16 \times 16$ using optimization procedure described in \cref{ssec:dl_implementation}. During the training we aim for average sparsity level $\frac{3}{512} \approx 0.6 \%$. During the reconstruction we set $\regparamrec_1 = 50$ and $ \regparamrec_2 = 0.0016$ obtained by minimizing the validation error. The learned dictionary is visualized in \cref{app:dic_viz}.

First of all, we compare our method against \ac{FBP}, which is an analytic reconstruction method that computes a regularised approximate inverse. 
We use the (approximate) implementation of \ac{FBP} in \ac{ODL} \cite{odl} with the Hanning filter and a relative frequency cut-off 0.75.

\paragraph{Variational models} 
Our comparison includes reconstruction by several variational models.
All models can be defined as a solution to  
\begin{equation}\label{eq:var_framework}
    \estim{\signal} = \argmin_{\signal \in \RecSpace} \DataLogLikelihood\bigl(\ForwardOp(\signal), \data\bigr) + \regparamrec \RegFunc(\signal)
\end{equation}
with $\DataLogLikelihood \colon \DataSpace \times \DataSpace \to \Real_+$ as in \cref{eq:WeightedL2} and with different choices of $\RegFunc \colon \RecSpace \to \Real$.

The hyper-parameter $\regparamrec$ is a regularization parameter that governs the trade-off between stability and the need to fit data.
Ideally the reconstruction methods includes a parameter selection rule for setting it, typically based on the noise level in data. 
These do not necessarily ensure best performance, so in order to ensure the different methods are compared at their best performance, we set this parameter empirically against validation data as to optimise performance. 
This also applies to the number of iterations used in the optimization scheme for solving \cref{eq:var_framework}.
A discussion on the potential regularising property of early stopping for variational methods is given in \cite{Effland:2020aa}.

One variational model is \ac{TV} \cite{rudin1992nonlinear}, which corresponds to choosing $\RegFunc(\signal) =  | \nabla \signal |_1$ in \cref{eq:var_framework} with $\nabla$ denoting the spatial gradient operator. 
The optimal choice of the regularisation parameter is $\regparamrec = 3 \cdot 10^{-4}$.
We then solve \cref{eq:var_framework} using the \ac{ADMM} algorithm with 1\,000 iterations.

The next variational model is \ac{TGV}, which was introduced partly to address some of the drawbacks (like stair-casing) that comes with using \ac{TV} \cite{bredies2010total,bredies2015tgv}.
\Ac{TGV} corresponds to selecting 
\begin{equation}\label{eq:TGVfunc}
     \RegFunc(\signal) = \min_z ||\nabla x - z||_1 + \alpha ||\mathcal{E} z||_1
\end{equation}
in \cref{eq:var_framework} where $\mathcal{E}$ is the symmetrized gradient operator, see \cite{bredies2015tgv} for the details. 
The optimal choice of regularisation parameters in our case is $\regparamrec = 6 \cdot 10^{-4}$ and $\alpha = 0.5$.
The same optimization procedure as for \ac{TV} is then used to solve \cref{eq:var_framework} with \cref{eq:TGVfunc}.

Finally, we also include the commonly used Huber regularization, which is a smooth relaxation of \ac{TV} regularization that replaces the non-smooth $l_1$-norm with the smooth Huber functional \cite{erdougan1999statistical}. 
This corresponds to $\RegFunc(\signal) = \mathcal{H}_{\gamma}(\nabla \signal)$ in \cref{eq:var_framework} with $\mathcal{H}_{\gamma} \colon \Real^{2\signaldim} \to \Real_+$ (Huber functional) defined as
\begin{equation}\label{eq:HuberReg}
    \mathcal{H}_{\gamma}(z) = \sum_{i=1}^{2\signaldim}  \frac{z_i^2}{2\gamma}\mathbf{1}_{|z_i | < \gamma} +\sum_{i=1}^{2\signaldim} \left(|z_i| -\frac{\gamma}{2} \right)\mathbf{1}_{|z_i | \geq \gamma}
\end{equation}
The optimal choice of regularisation parameters for our case is $\regparamrec = 5 \cdot 10^{-4}$ and $\gamma =4 \cdot 10^{-4}$.
The optimization \cref{eq:var_framework} with \cref{eq:HuberReg} is solved using Nesterov's accelerated gradient descent \cite{Nesterov:1983} with 70~iterations. 
\paragraph{Deep learning based methods}  
\Ac{DIP} relies on the fact that merely a representation of a signal as an output of a convolutional neural network $\signal = f_{\theta}(\dictcoeff)$ is sufficient for regularization in inverse problems \cite{ulyanov2018deep, baguer2020computed}. Here, we use an architecture of a neural network suggested by \cite{ulyanov2018deep} for denoising to perform image reconstruction by minimizing
\begin{equation}\label{eq:DIPtrain}
\min_{\theta} \DataLogLikelihood\bigl(\ForwardOp(f_{\theta}(\dictcoeff)), \data\bigr)    
\end{equation}
with respect to parameters of the neural network. The latent variable $z = z_0 + \Delta z$, where $z_0$ is sampled only once (before the start of the optimization process) and $\Delta z$ is regularizing noise, which is sampled during every iteration for solving \cref{eq:DIPtrain}. 
This approach is completely unsupervised, since it does not require any training prior to the reconstruction. We use the same optimization procedure and hyper-parameters, with few exceptions. We reduce the standard deviation of the regularizing noise by a factor of 2 and increase the number of iterations to $6000$. We observe that the approach might slightly benefit from additional iterations an higher noise, however even in this setting it is already remarkably slow compared to the other methods. 

Another learned method that we used for comparison is \ac{AR} proposed in \cite{lunz2018adversarial}. The method uses variational formulation \cref{eq:var_framework}, where the regularizing component $\RegFunc(x)$ is a trained neural network. This method is can be seen as a semi-supervised, since it is trained using samples of both, reconstructions and projection data. However it does not rely on samples being coupled, i.e.\ correspond to the same scan. We use a default convolutional architecture and set regularization parameter $\regparamrec$ as suggested by the authors. The optimizations is done for 2000 iterations. 

Finally, \ac{LPD} is a supervised deep learning method that approximately computes the posterior mean image given measured data, so it has no regularization parameter.
It has shown state-of-the-art performance \cite{adler2018learned} in low-dose \ac{CT} reconstruction and we include a version of this method applied to log-data.

We have paid a reasonable amount of effort to improve the performance of all schemes in the particular test setting. In particular, we re-scaled the images for the deep learning methods so that the gray-scale values approximately fit the interval $[0,1]$. This is important, because the initialization of the neural networks is traditionally adapted for this scaling. Unsurprisingly, the architectures and hyper-parameters proposed by the authors in most cases provided the optimal performance. This observation is supported by the fact that \cite{lunz2018adversarial} and \cite{adler2018learned} has been originally evaluated on the data simulated from same image dataset.

\subsection{Results}

Quantitative performance results in terms of \ac{PSNR} and \ac{SSIM} are given in \cref{tab:results}. In both metrics, \ac{LPD} performs best. Second best is \ac{DL} followed by \ac{AR}. \Ac{FBP} performs worst in both metrics among the tested algorithms. 

\Cref{fig:slice36,fig:slice163} show exemplary slices from abdomen and pelvis, respectively, for a qualitative assessment of the image quality.

\Cref{fig:slice36} highlights two region of interests, which are well suited to assess the noise appearance and the delineation of low-contrast structures. 

With respect to noise texture, \ac{TV} shows the well-known behaviour with unnaturally looking flat regions with remaining salt-and-pepper noise. \Ac{TGV} shows a horizontally structured noise pattern, which is also present (but less pronounced) in the image with Huber regularization. AR shows a rather natural noise pattern, except for the aspect that it appears more low-frequent than the noise pattern in ground truth and \ac{FBP} images. The most natural noise pattern is achieved by DIP and \ac{LPD}. Finally, in \ac{DL} barely any noise is remaining.

Regarding the low-contrast delineation, we focus on the boundaries of the pancreas in the lower right zoomed image region. Here, \ac{TV} shows the well-known scruffy edges. The edge appearance improves in \ac{TGV} at the cost of a higher noise level. Edges appear visibly blurred in DIP and AR, which is consistent for AR with the observation of the presence of rather low-frequent noise. Pancreas delineation is best for Huber, \ac{DL} and \ac{LPD}. 

\Cref{fig:slice163} highlights the performance of the algorithms on low-contrast structures in the fatty tissue (bottom left zoom) high-contrast structures (bottom right zoom).

With respect to the structures in the fatty tissue, \ac{TV} appears patchy and unnaturally looking. All other methods show a comparable level of details. More differences between the algorithms are present in the bony structure inside the hip joint shown in the zoom in the bottom right. Note the arc-shaped bony structure in the center of the ground truths (\cref{fig:slice163}a). This structure is not recovered at all in DIP and \ac{AR}, while it is only partially visible in AR, Huber, \ac{TV}, \ac{TGV}. Only \ac{LPD} and \ac{DL} properly recover this structure.

In summary, the qualitative image quality assessment supports the quantitative evaluation that \ac{LPD} performs best followed by \ac{DL}.

\begin{figure}
\centering
\begin{subfigure}{0.3\linewidth}
\begin{tikzpicture}[spy using outlines={
  rectangle, 
  red, 
  magnification=3,
  size=0.4\linewidth, 
  connect spies}]
 \node{\includegraphics[width=\linewidth]{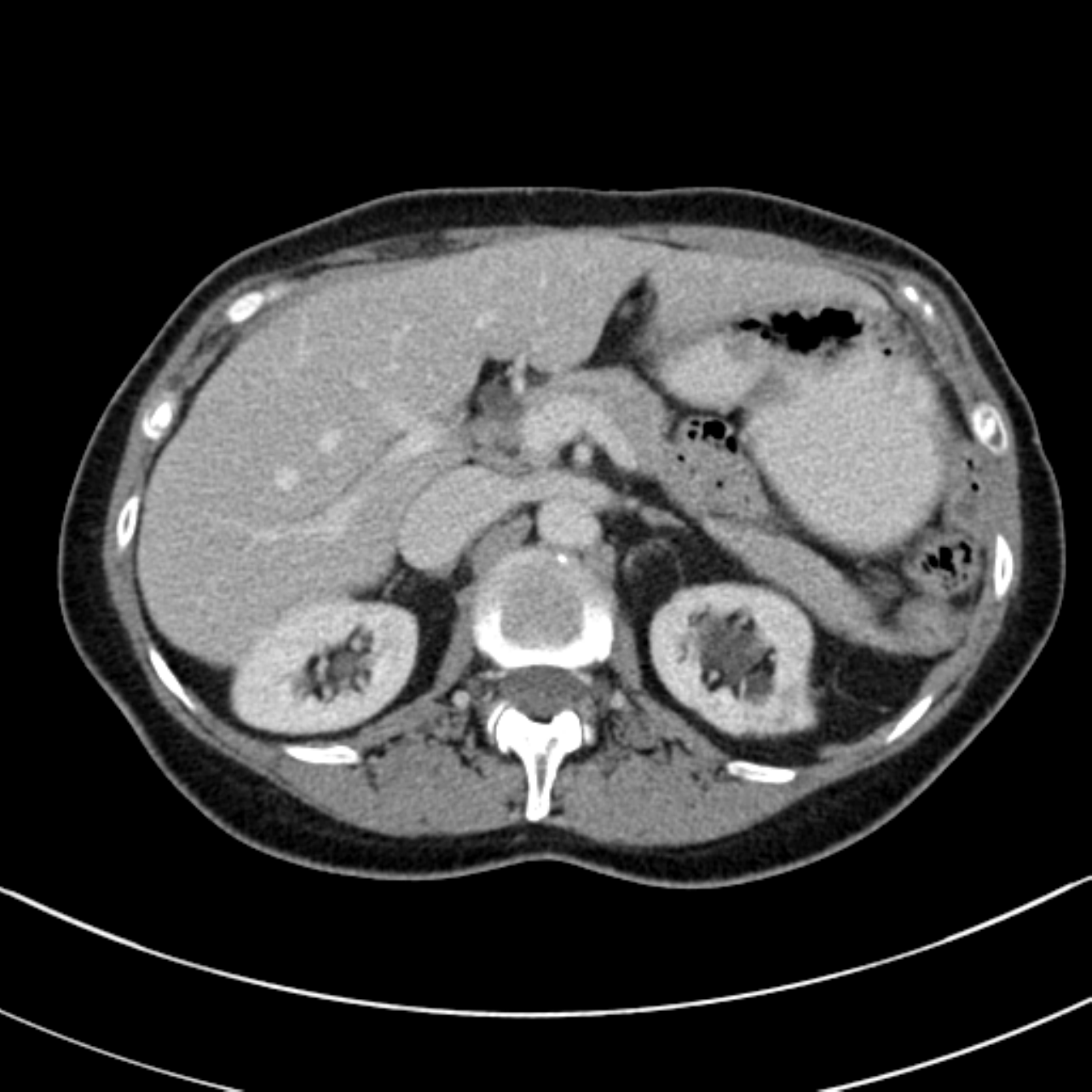}};
 \spy on (-0.35,0.2) in node [left] at%
   (0.5\linewidth,-0.3\linewidth);
 \spy on (-0.9,-0.17) in node [left] at%
   (-0.1\linewidth,0.4\linewidth);   
\end{tikzpicture}
\caption{True}
\end{subfigure}%
\hfill
\begin{subfigure}{0.3\textwidth}
\begin{tikzpicture}[spy using outlines={
  rectangle, 
  red, 
  magnification=3,
  size=0.4\linewidth, 
  connect spies}]
 \node{\includegraphics[width=\linewidth]{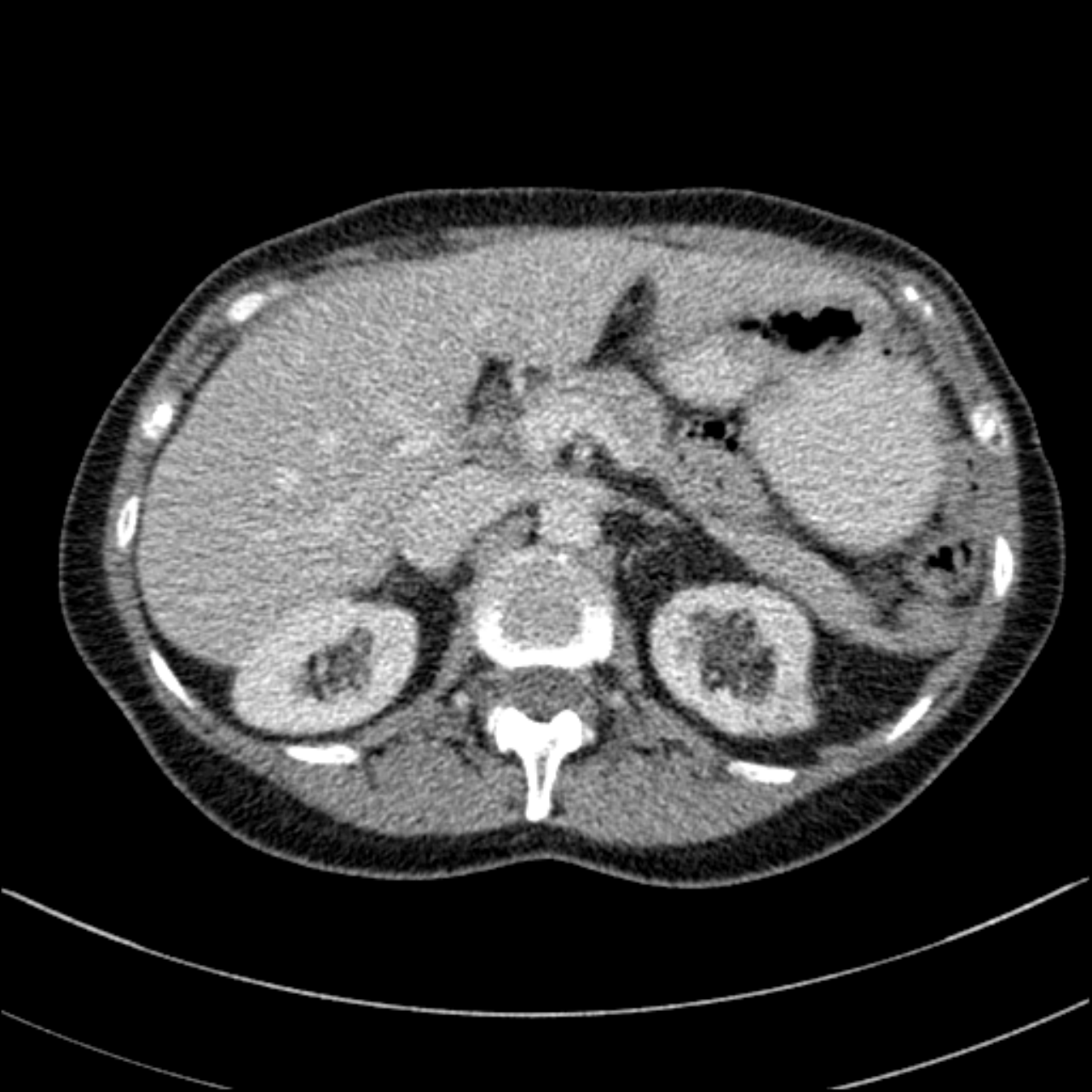}};
 \spy on (-0.35,0.2) in node [left] at%
   (0.5\linewidth,-0.3\linewidth);
 \spy on (-0.9,-0.17) in node [left] at%
   (-0.1\linewidth,0.4\linewidth);   
\end{tikzpicture}
\caption{FBP}
\end{subfigure}%
\hfill
\begin{subfigure}{0.3\textwidth}
\begin{tikzpicture}[spy using outlines={
  rectangle, 
  red, 
  magnification=3,
  size=0.4\linewidth, 
  connect spies}]
 \node{\includegraphics[width=\linewidth]{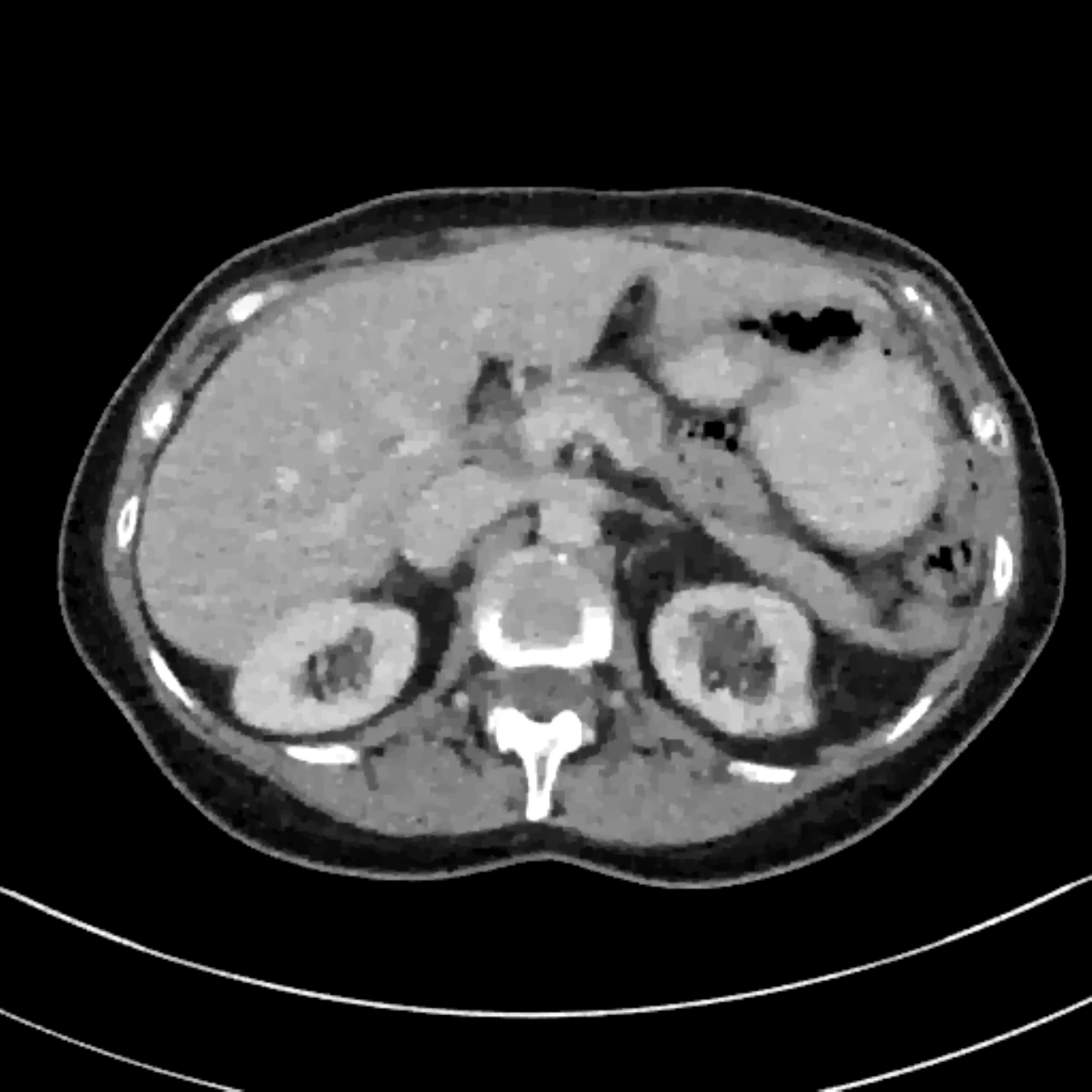}};
 \spy on (-0.35,0.2) in node [left] at%
   (0.5\linewidth,-0.3\linewidth);
 \spy on (-0.9,-0.17) in node [left] at%
   (-0.1\linewidth,0.4\linewidth);   
\end{tikzpicture}
\caption{TV}
\end{subfigure}%
\\[1em]
\begin{subfigure}{0.3\textwidth}
\begin{tikzpicture}[spy using outlines={
  rectangle, 
  red, 
  magnification=3,
  size=0.4\linewidth, 
  connect spies}]
 \node{\includegraphics[width=\linewidth]{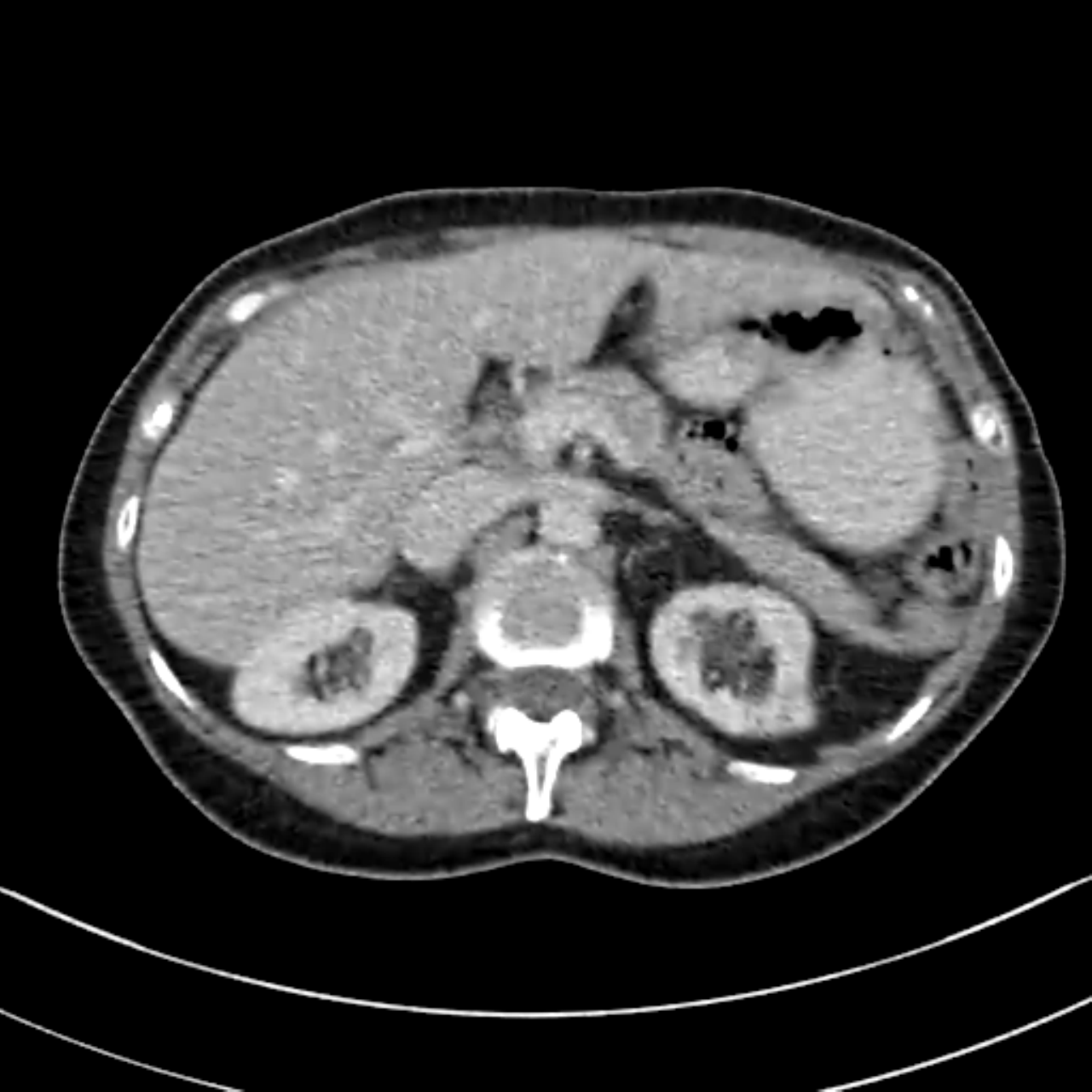}};
 \spy on (-0.35,0.2) in node [left] at%
   (0.5\linewidth,-0.3\linewidth);
 \spy on (-0.9,-0.17) in node [left] at%
   (-0.1\linewidth,0.4\linewidth);   
\end{tikzpicture}
\caption{TGV}
\end{subfigure}%
\hfill
\begin{subfigure}{0.3\textwidth}
\begin{tikzpicture}[spy using outlines={
  rectangle, 
  red, 
  magnification=3,
  size=0.4\linewidth, 
  connect spies}]
 \node{\includegraphics[width=\linewidth]{images/Huber_36}};
 \spy on (-0.35,0.2) in node [left] at%
   (0.5\linewidth,-0.3\linewidth);
 \spy on (-0.9,-0.17) in node [left] at%
   (-0.1\linewidth,0.4\linewidth);   
\end{tikzpicture}
\caption{Huber}
\end{subfigure}
\hfill 
\begin{subfigure}{0.3\textwidth}
\begin{tikzpicture}[spy using outlines={
  rectangle, 
  red, 
  magnification=3,
  size=0.4\linewidth, 
  connect spies}]
 \node{\includegraphics[width=\linewidth]{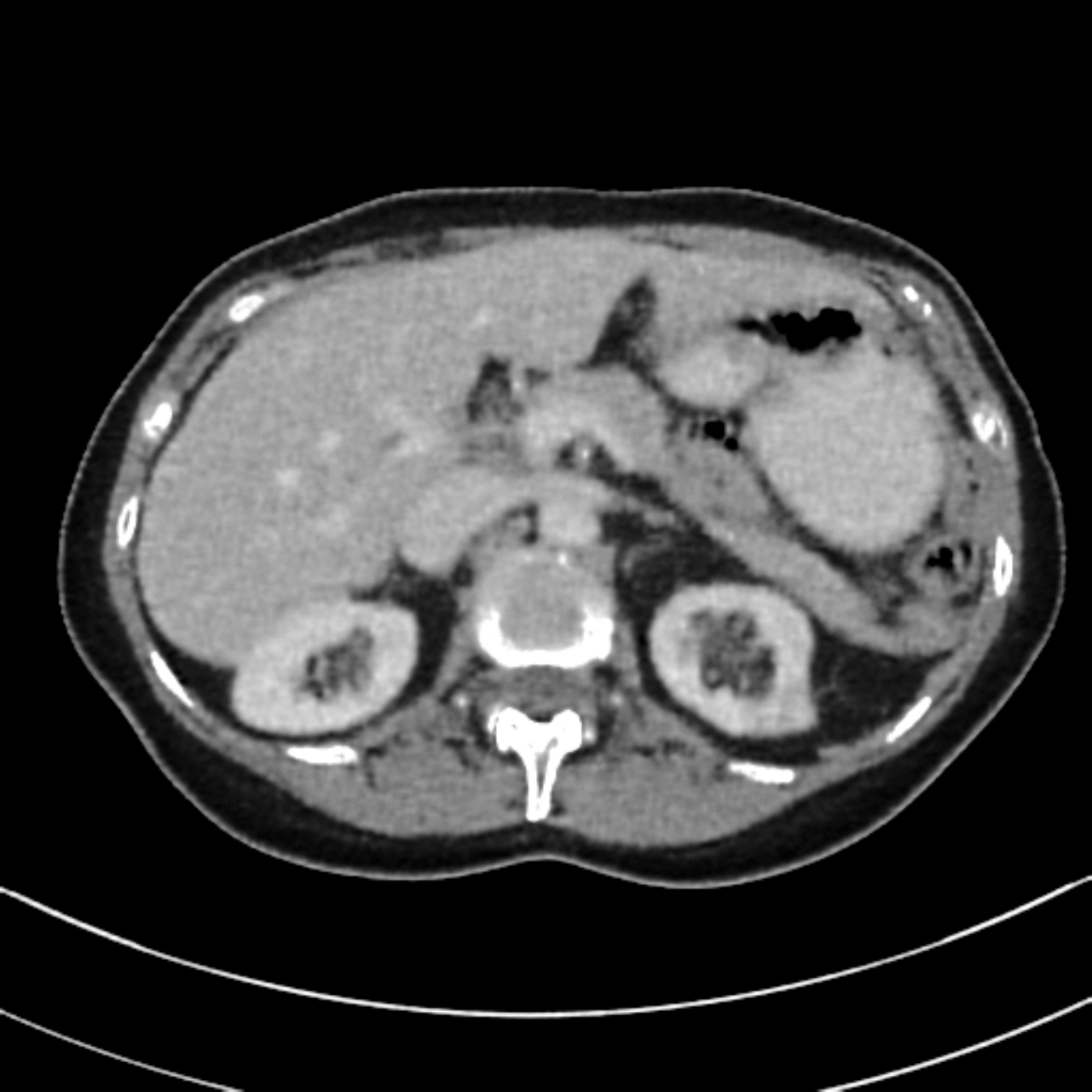}};
 \spy on (-0.35,0.2) in node [left] at%
   (0.5\linewidth,-0.3\linewidth);
 \spy on (-0.9,-0.17) in node [left] at%
   (-0.1\linewidth,0.4\linewidth);   
\end{tikzpicture}
\caption{DIP}
\end{subfigure}
\\[1em]
\begin{subfigure}{0.3\textwidth}
\begin{tikzpicture}[spy using outlines={
  rectangle, 
  red, 
  magnification=3,
  size=0.4\linewidth, 
  connect spies}]
 \node{\includegraphics[width=\linewidth]{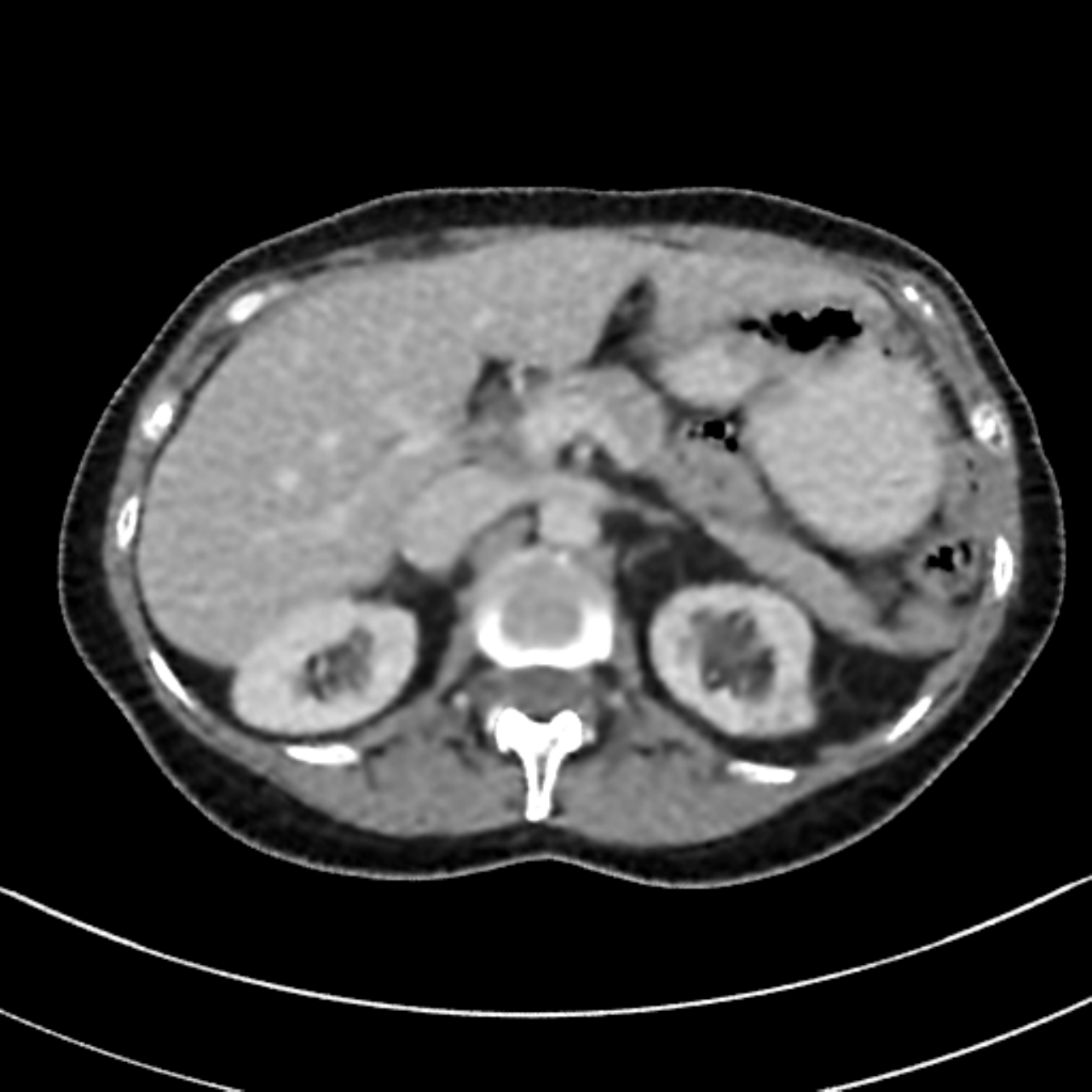}};
 \spy on (-0.35,0.2) in node [left] at%
   (0.5\linewidth,-0.3\linewidth);
 \spy on (-0.9,-0.17) in node [left] at%
   (-0.1\linewidth,0.4\linewidth);   
\end{tikzpicture}
\caption{AR}
\end{subfigure}%
\hfill
\begin{subfigure}{0.3\textwidth}
\begin{tikzpicture}[spy using outlines={
  rectangle, 
  red, 
  magnification=3,
  size=0.4\linewidth, 
  connect spies}]
 \node{\includegraphics[width=\linewidth]{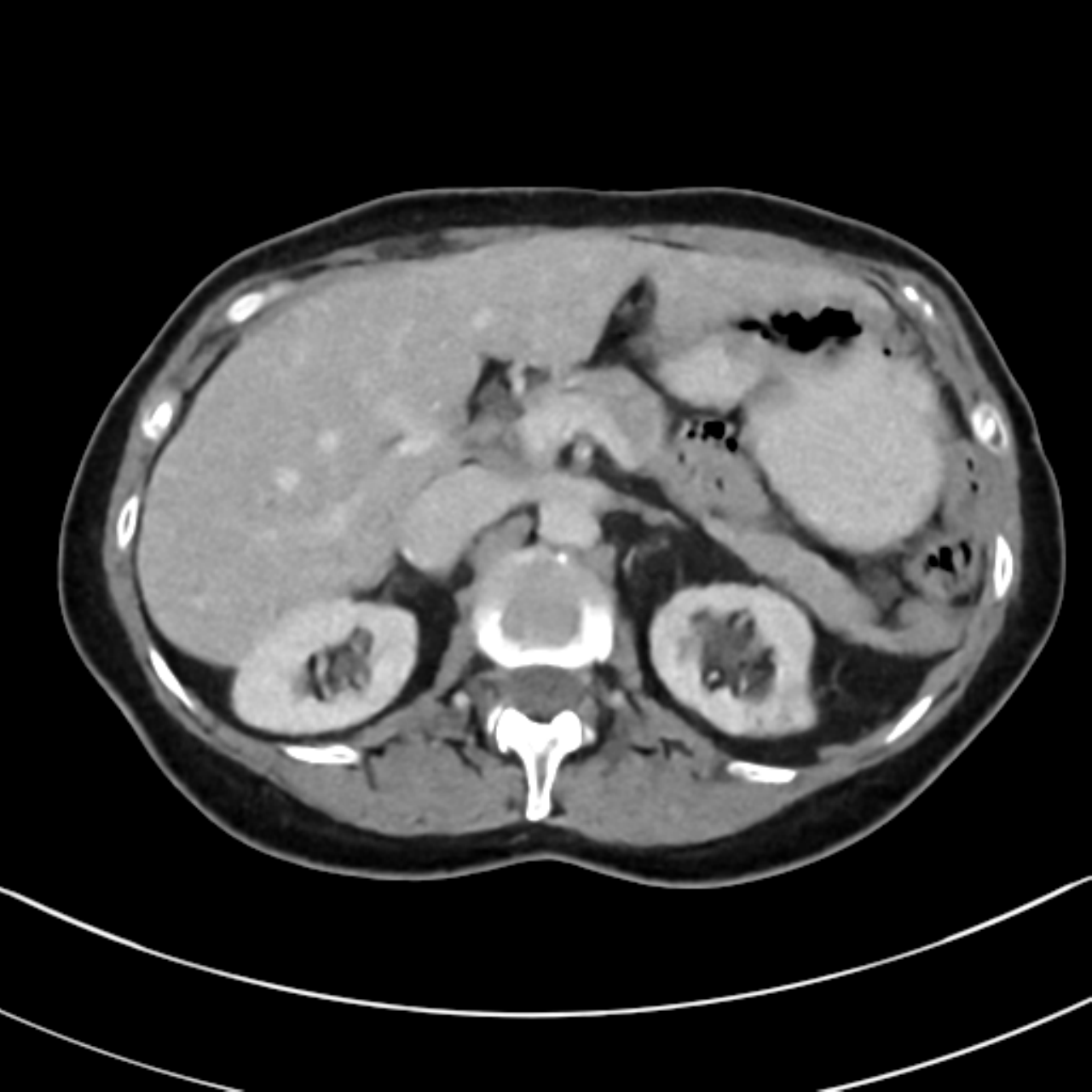}};
 \spy on (-0.35,0.2) in node [left] at%
   (0.5\linewidth,-0.3\linewidth);
 \spy on (-0.9,-0.17) in node [left] at%
   (-0.1\linewidth,0.4\linewidth);   
\end{tikzpicture}
\caption{LPD}
\end{subfigure}
\hfill
\begin{subfigure}{0.3\textwidth}
\begin{tikzpicture}[spy using outlines={
  rectangle, 
  red, 
  magnification=3,
  size=0.4\linewidth, 
  connect spies}]
 \node{\includegraphics[width=\linewidth]{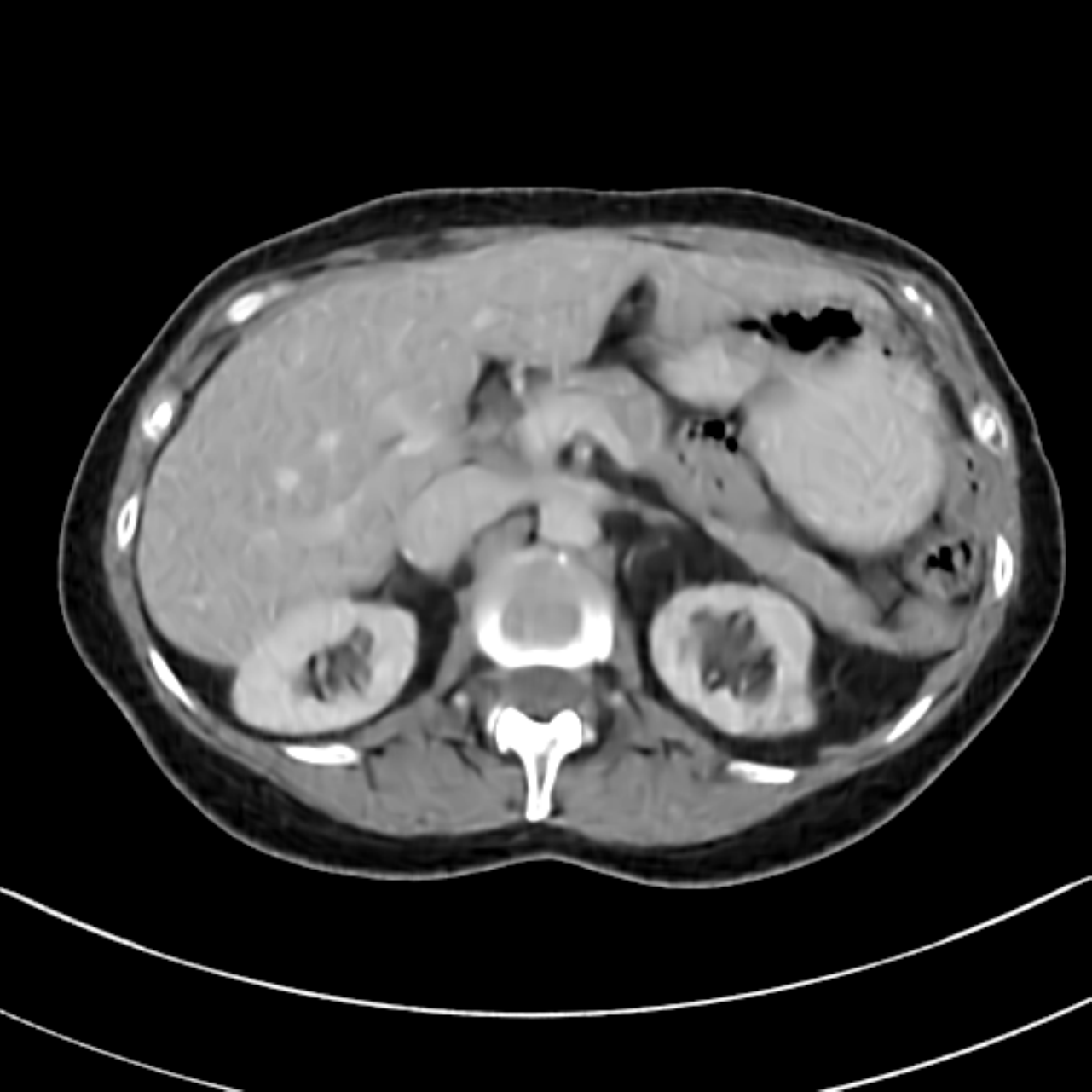}};
 \spy on (-0.35,0.2) in node [left] at%
   (0.5\linewidth,-0.3\linewidth);
 \spy on (-0.9,-0.17) in node [left] at%
   (-0.1\linewidth,0.4\linewidth);   
\end{tikzpicture}
\caption{DL}
\end{subfigure}%
\caption{Example slice of the abdomen (level 80 HU, window 370 HU). \Acp{ROI} illustrate noise texture (top left) and delineation of low contrast structures (liver-kidney border in the top left \ac{ROI} and liver-pancreas transition as well as contrasted vessel in the liver in the bottom right \ac{ROI}.}
\label{fig:slice36}
\end{figure}

\begin{figure}
\centering
\begin{subfigure}{0.3\textwidth}
\begin{tikzpicture}[spy using outlines={
  rectangle, 
  red, 
  magnification=3,
  size=0.4\linewidth, 
  connect spies}]
 \node{\includegraphics[width=\linewidth]{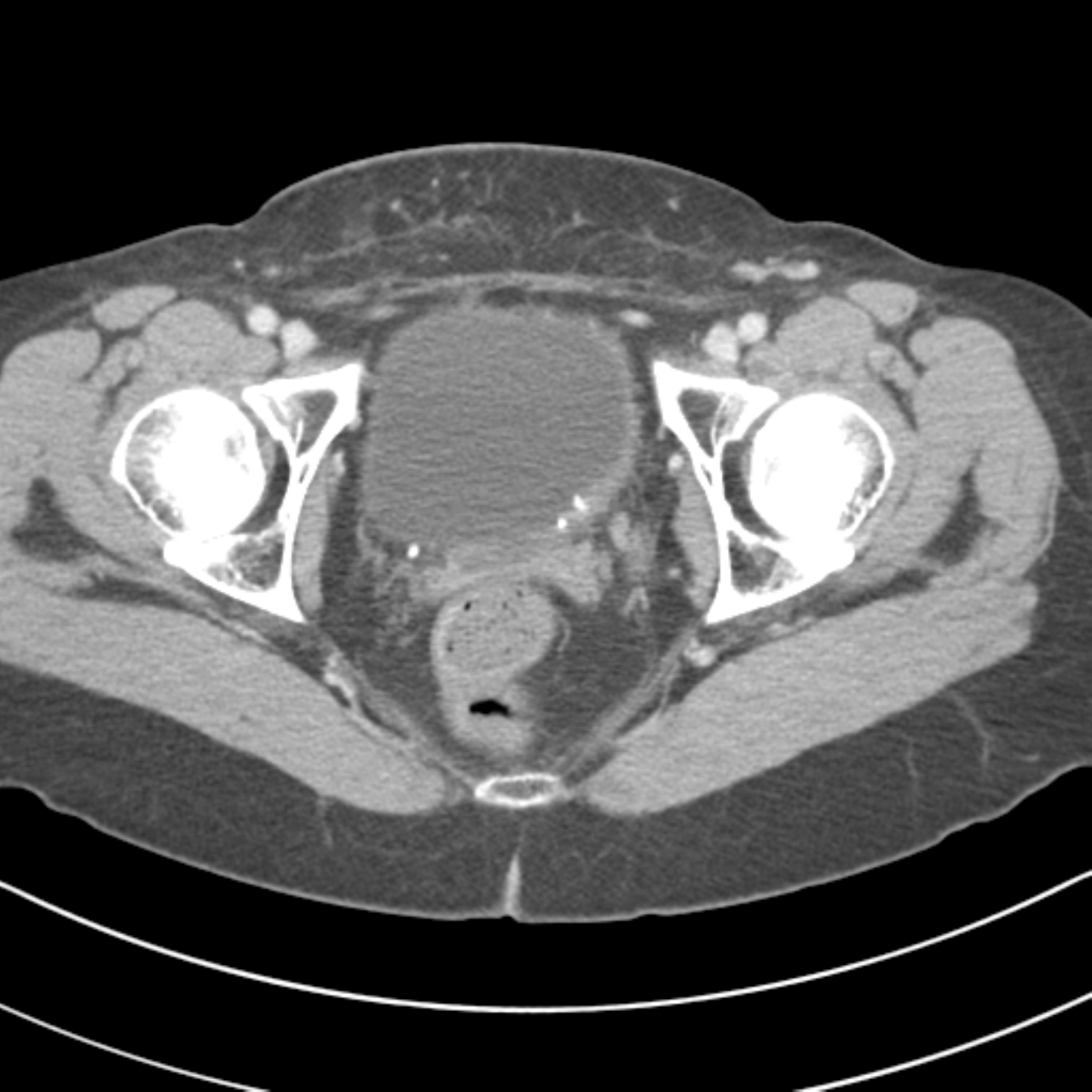}};
 \spy on (-0.55,1) in node [left] at%
   (-0.1\textwidth,-0.3\textwidth);
 \spy on (0.55,0.47) in node [left] at%
   (0.5\linewidth,-0.3\linewidth);   
\end{tikzpicture}
\caption{True}
\end{subfigure}%
\hfill
\begin{subfigure}{0.3\textwidth}
\begin{tikzpicture}[spy using outlines={
  rectangle, 
  red, 
  magnification=3,
  size=0.4\linewidth, 
  connect spies}]
 \node{\includegraphics[width=\linewidth]{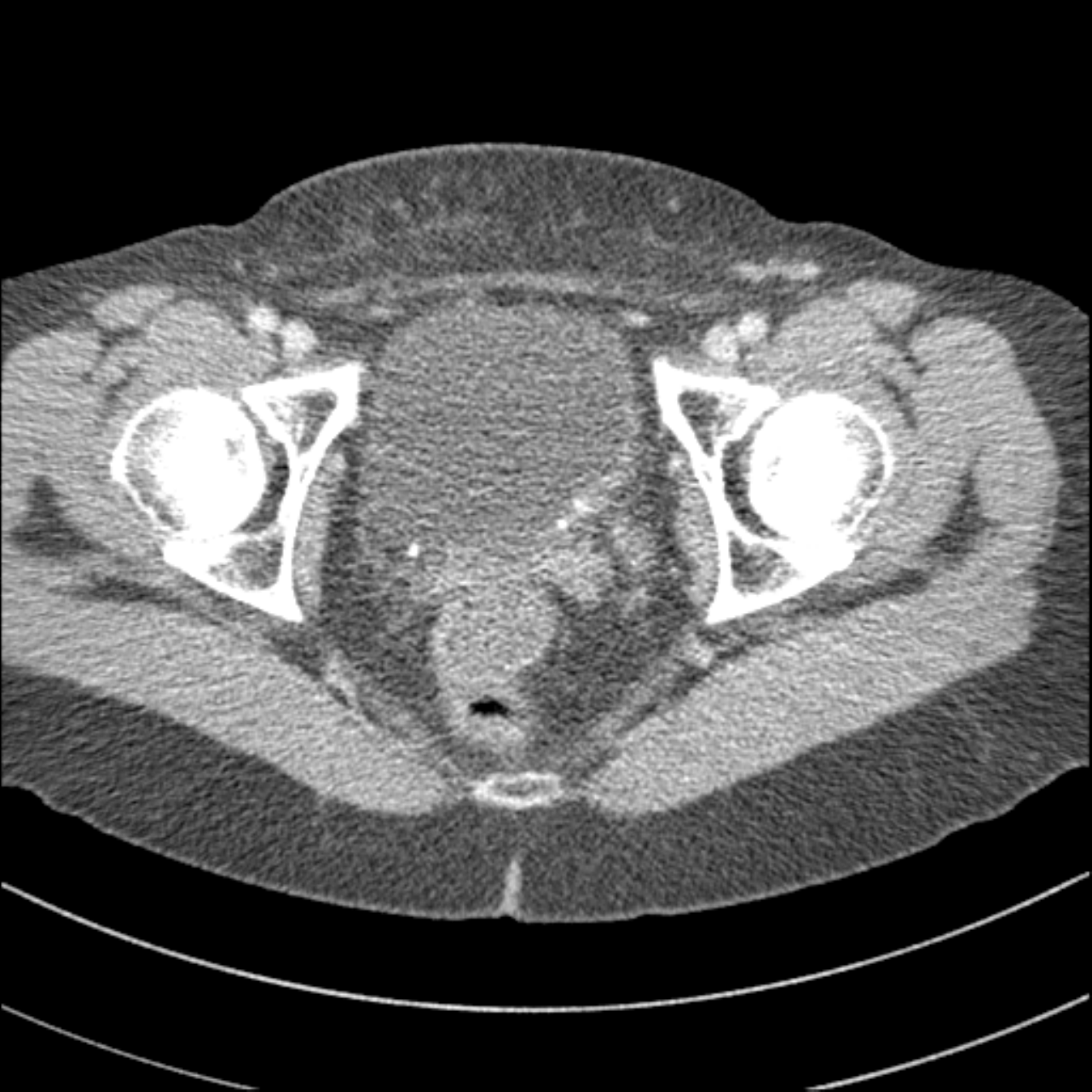}};
 \spy on (-0.55,1) in node [left] at%
   (-0.1\textwidth,-0.3\textwidth);
 \spy on (0.55,0.47) in node [left] at%
   (0.5\linewidth,-0.3\linewidth);   
\end{tikzpicture}
\caption{FBP}
\end{subfigure}%
\hfill
\begin{subfigure}{0.3\textwidth}
\begin{tikzpicture}[spy using outlines={
  rectangle, 
  red, 
  magnification=3,
  size=0.4\linewidth, 
  connect spies}]
 \node{\includegraphics[width=\linewidth]{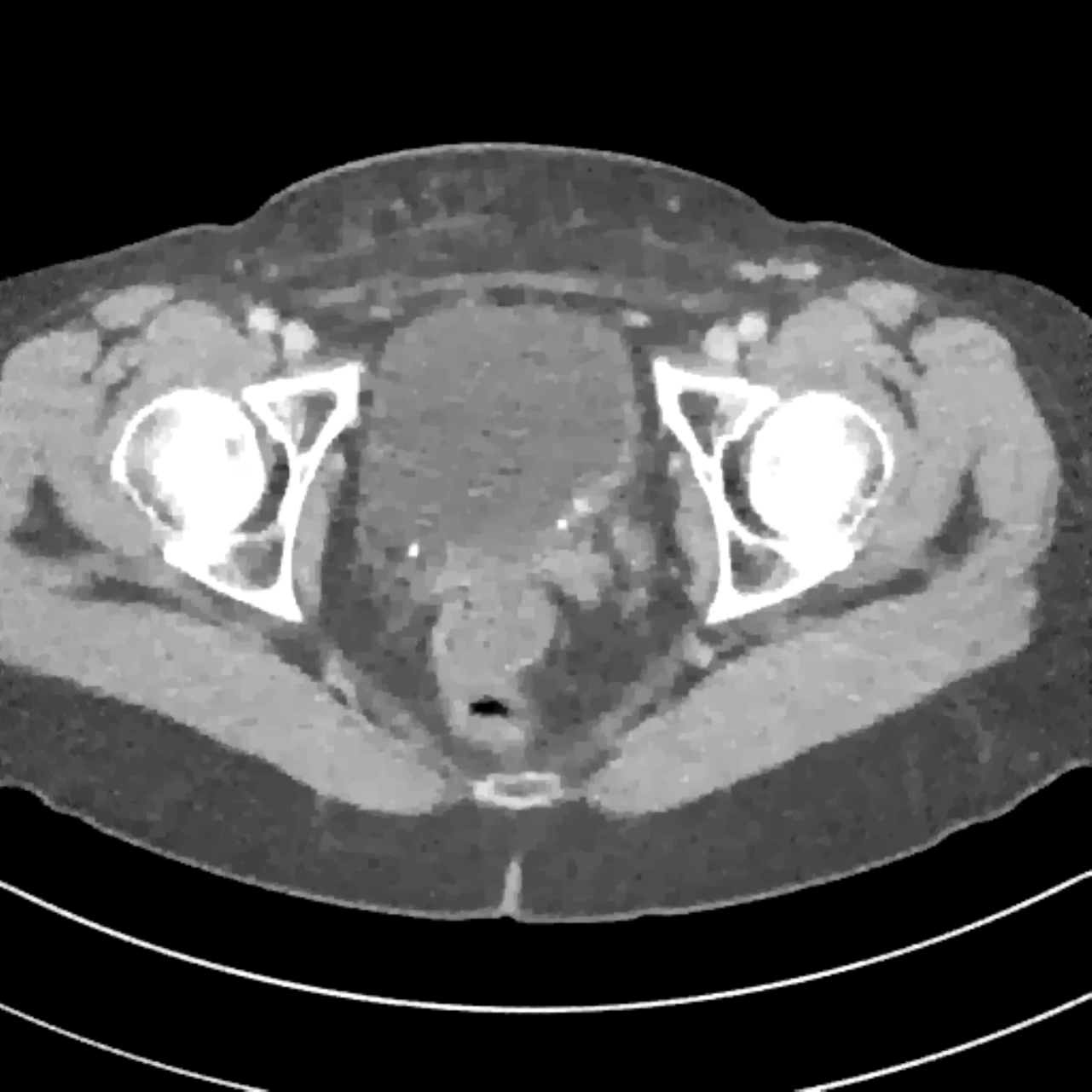}};
 \spy on (-0.55,1) in node [left] at%
   (-0.1\textwidth,-0.3\textwidth);
 \spy on (0.55,0.47) in node [left] at%
   (0.5\linewidth,-0.3\linewidth);   
\end{tikzpicture}
\caption{TV}
\end{subfigure}
\\
\begin{subfigure}{0.3\textwidth}
\begin{tikzpicture}[spy using outlines={
  rectangle, 
  red, 
  magnification=3,
  size=0.4\linewidth, 
  connect spies}]
 \node{\includegraphics[width=\linewidth]{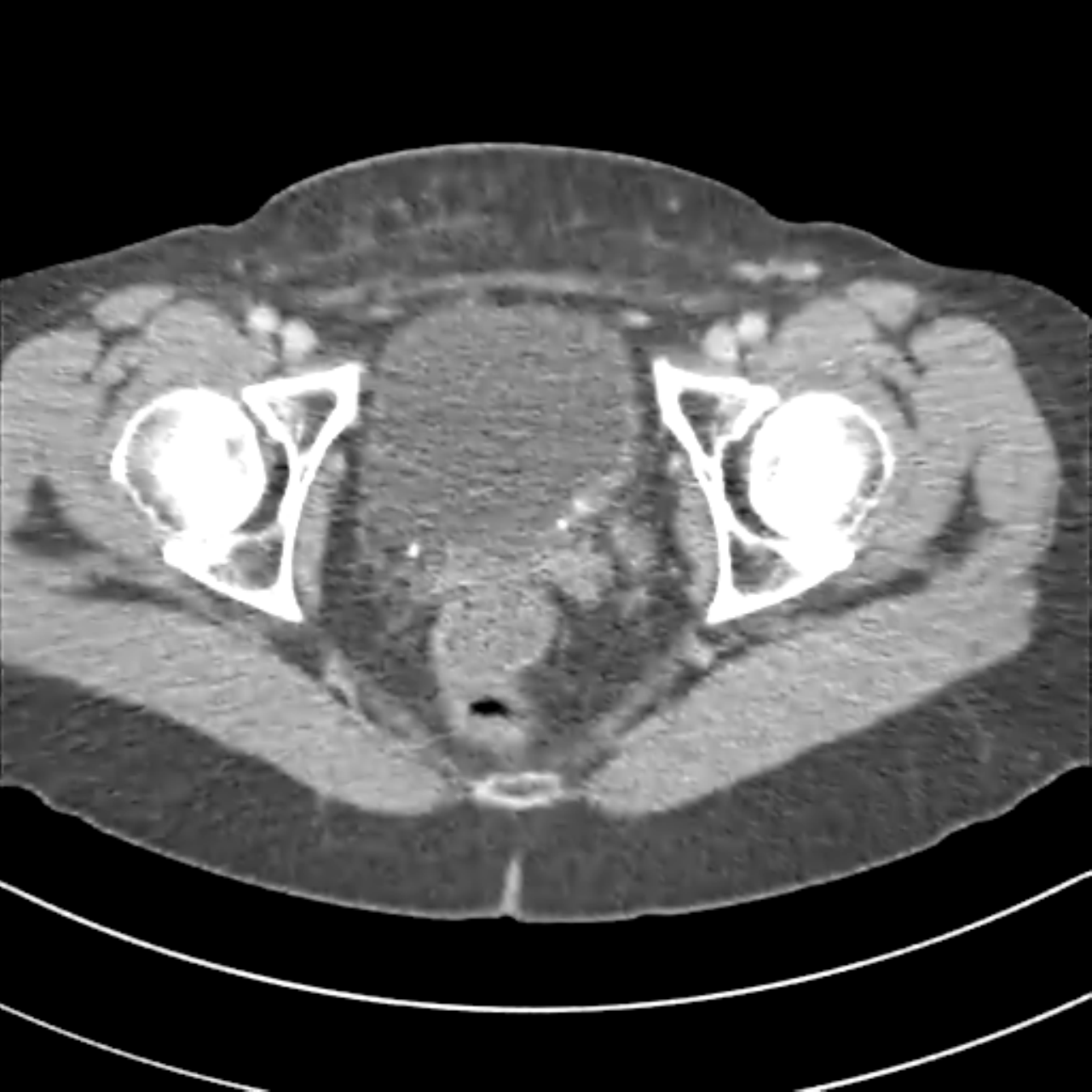}};
 \spy on (-0.55,1) in node [left] at%
   (-0.1\textwidth,-0.3\textwidth);
 \spy on (0.55,0.47) in node [left] at%
   (0.5\linewidth,-0.3\linewidth);   
\end{tikzpicture}
\caption{TGV}
\end{subfigure}%
\hfill 
\begin{subfigure}{0.3\textwidth}
\begin{tikzpicture}[spy using outlines={
  rectangle, 
  red, 
  magnification=3,
  size=0.4\linewidth, 
  connect spies}]
 \node{\includegraphics[width=\linewidth]{images/Huber_163}};
 \spy on (-0.55,1) in node [left] at%
   (-0.1\textwidth,-0.3\textwidth);
 \spy on (0.55,0.47) in node [left] at%
   (0.5\linewidth,-0.3\linewidth);   
\end{tikzpicture}
\caption{Huber}
\end{subfigure}
\hfill
\begin{subfigure}{0.3\textwidth}
\begin{tikzpicture}[spy using outlines={
  rectangle, 
  red, 
  magnification=3,
  size=0.4\linewidth, 
  connect spies}]
 \node{\includegraphics[width=\linewidth]{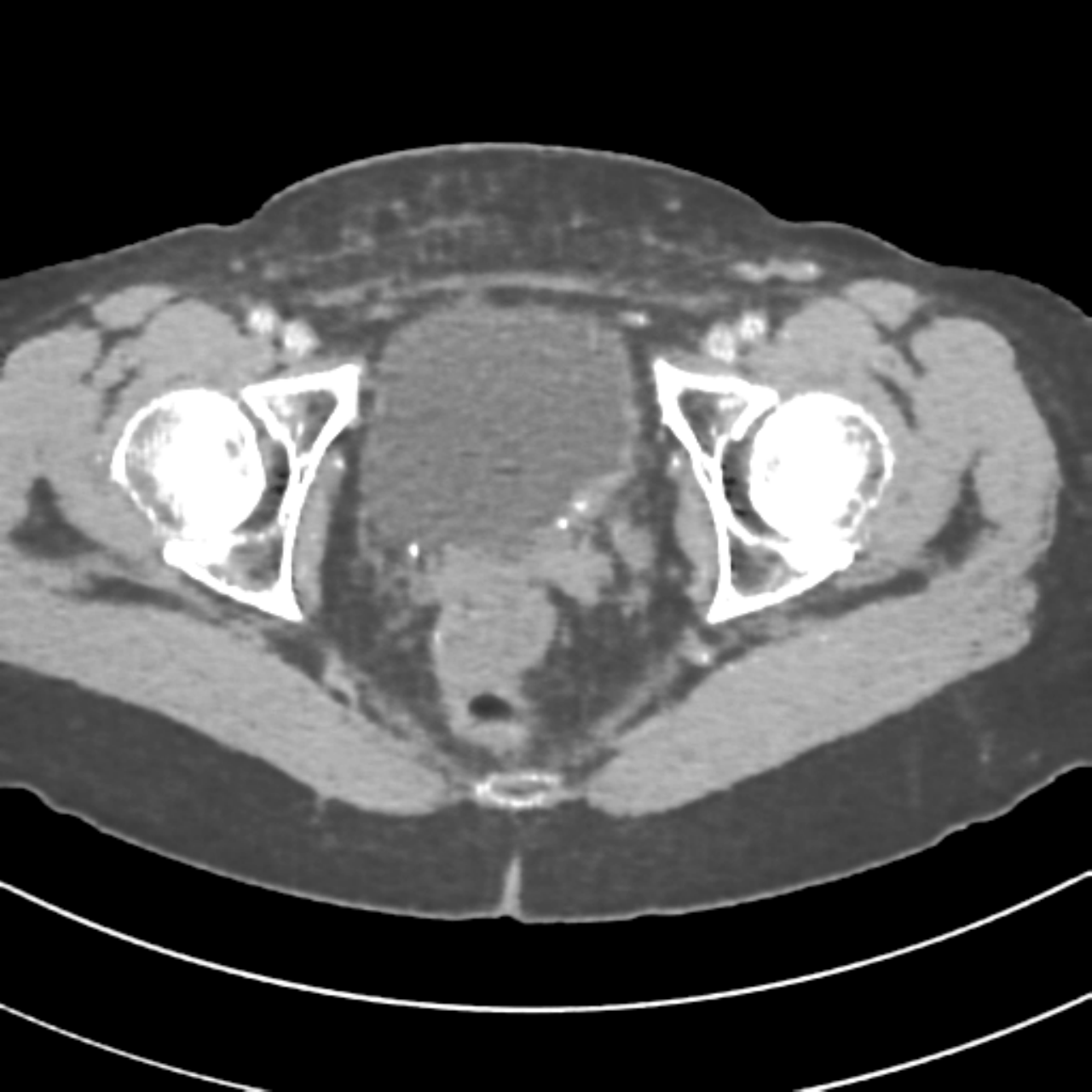}};
 \spy on (-0.55,1) in node [left] at%
   (-0.1\textwidth,-0.3\textwidth);
 \spy on (0.55,0.47) in node [left] at%
   (0.5\linewidth,-0.3\linewidth);   
\end{tikzpicture}
\caption{DIP}
\end{subfigure}
\\
\begin{subfigure}{0.3\textwidth}
\begin{tikzpicture}[spy using outlines={
  rectangle, 
  red, 
  magnification=3,
  size=0.4\linewidth, 
  connect spies}]
 \node{\includegraphics[width=\linewidth]{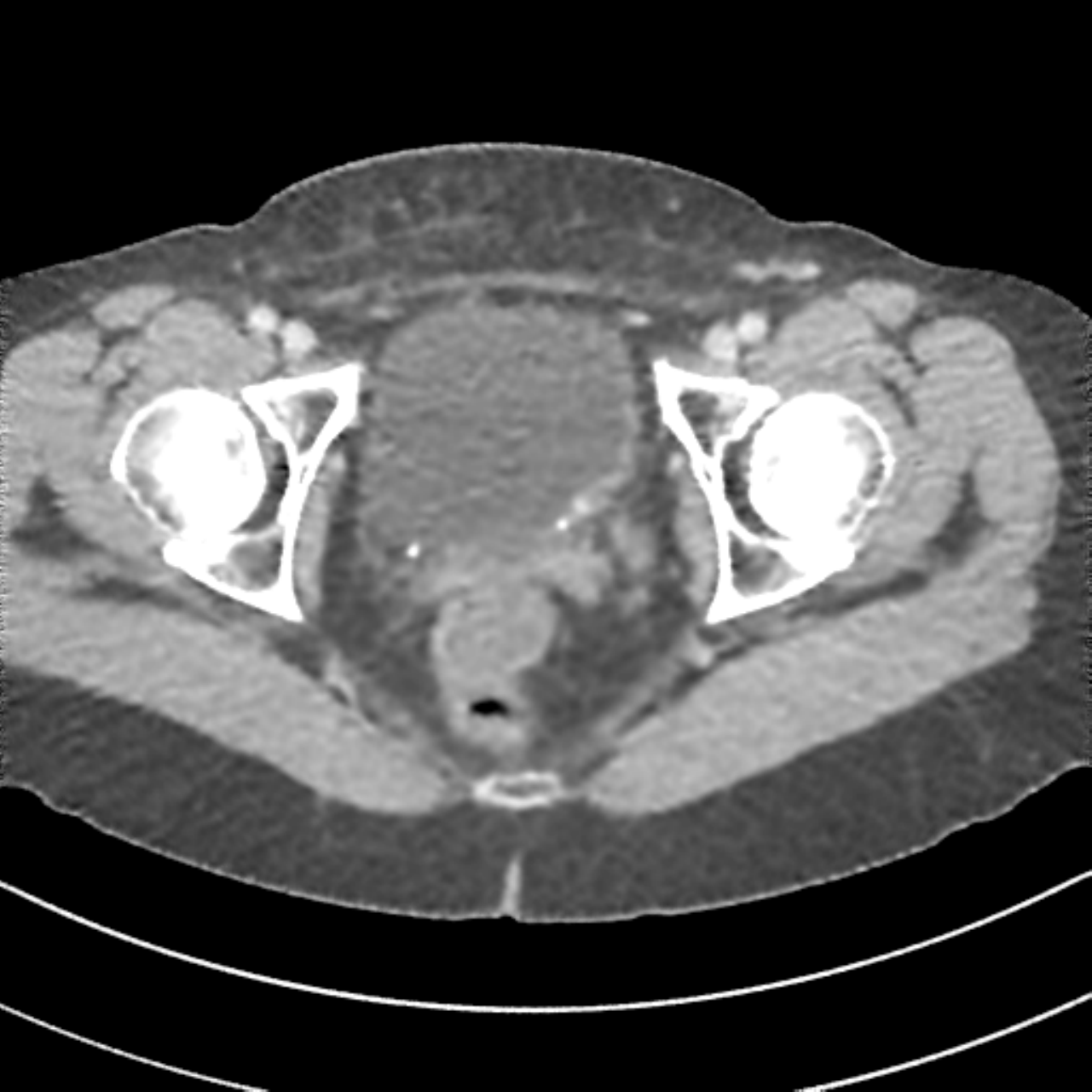}};
 \spy on (-0.55,1) in node [left] at%
   (-0.1\textwidth,-0.3\textwidth);
 \spy on (0.55,0.47) in node [left] at%
   (0.5\linewidth,-0.3\linewidth);   
\end{tikzpicture}
\caption{AR}
\end{subfigure}%
\hfill
\begin{subfigure}{0.3\textwidth}
\begin{tikzpicture}[spy using outlines={
  rectangle, 
  red, 
  magnification=3,
  size=0.4\linewidth, 
  connect spies}]
 \node{\includegraphics[width=\linewidth]{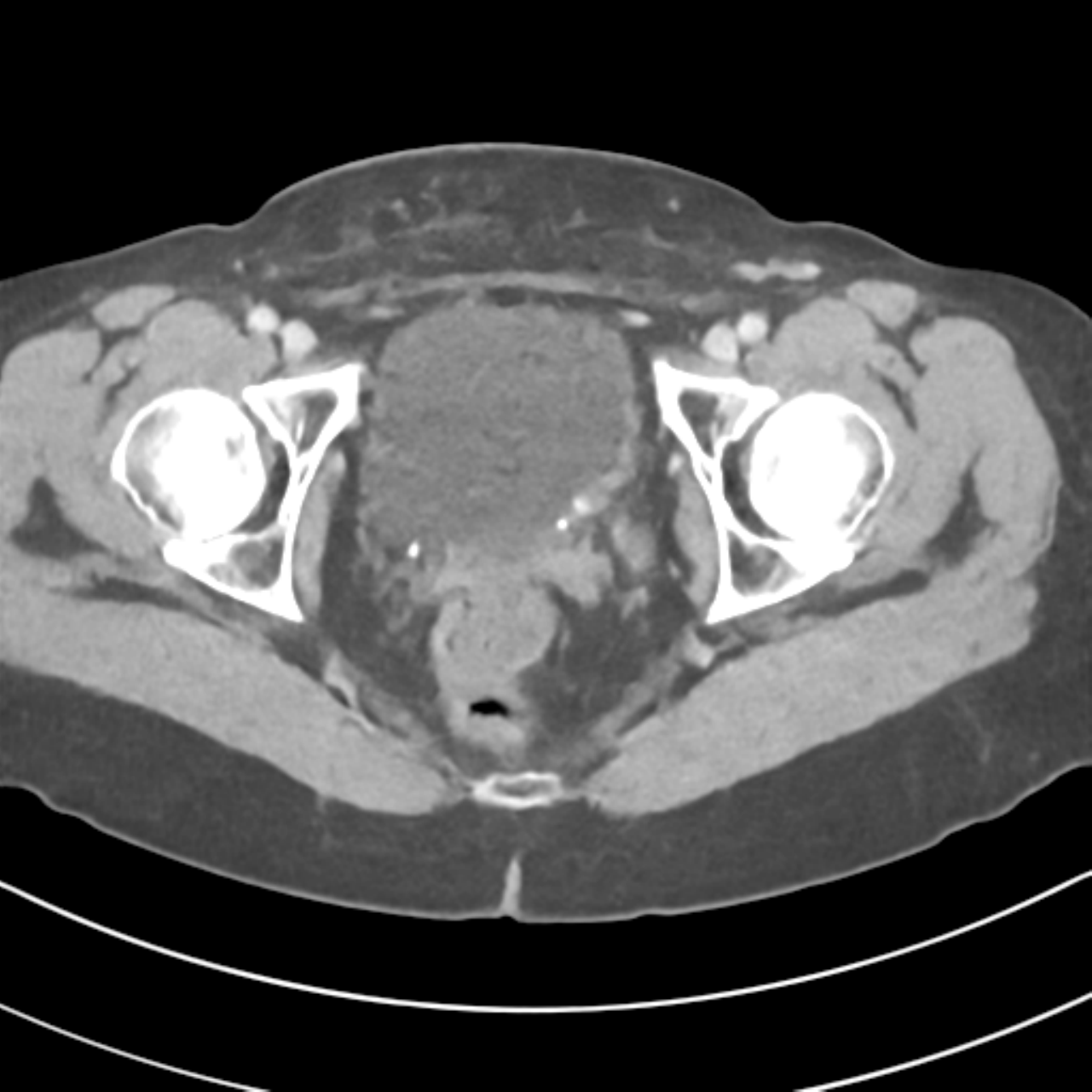}};
 \spy on (-0.55,1) in node [left] at%
   (-0.1\textwidth,-0.3\textwidth);
 \spy on (0.55,0.47) in node [left] at%
   (0.5\linewidth,-0.3\linewidth);   
\end{tikzpicture}
\caption{LPD}
\end{subfigure}
\hfill
\begin{subfigure}{0.3\textwidth}
\begin{tikzpicture}[spy using outlines={
  rectangle, 
  red, 
  magnification=3,
  size=0.4\linewidth, 
  connect spies}]
 \node{\includegraphics[width=\linewidth]{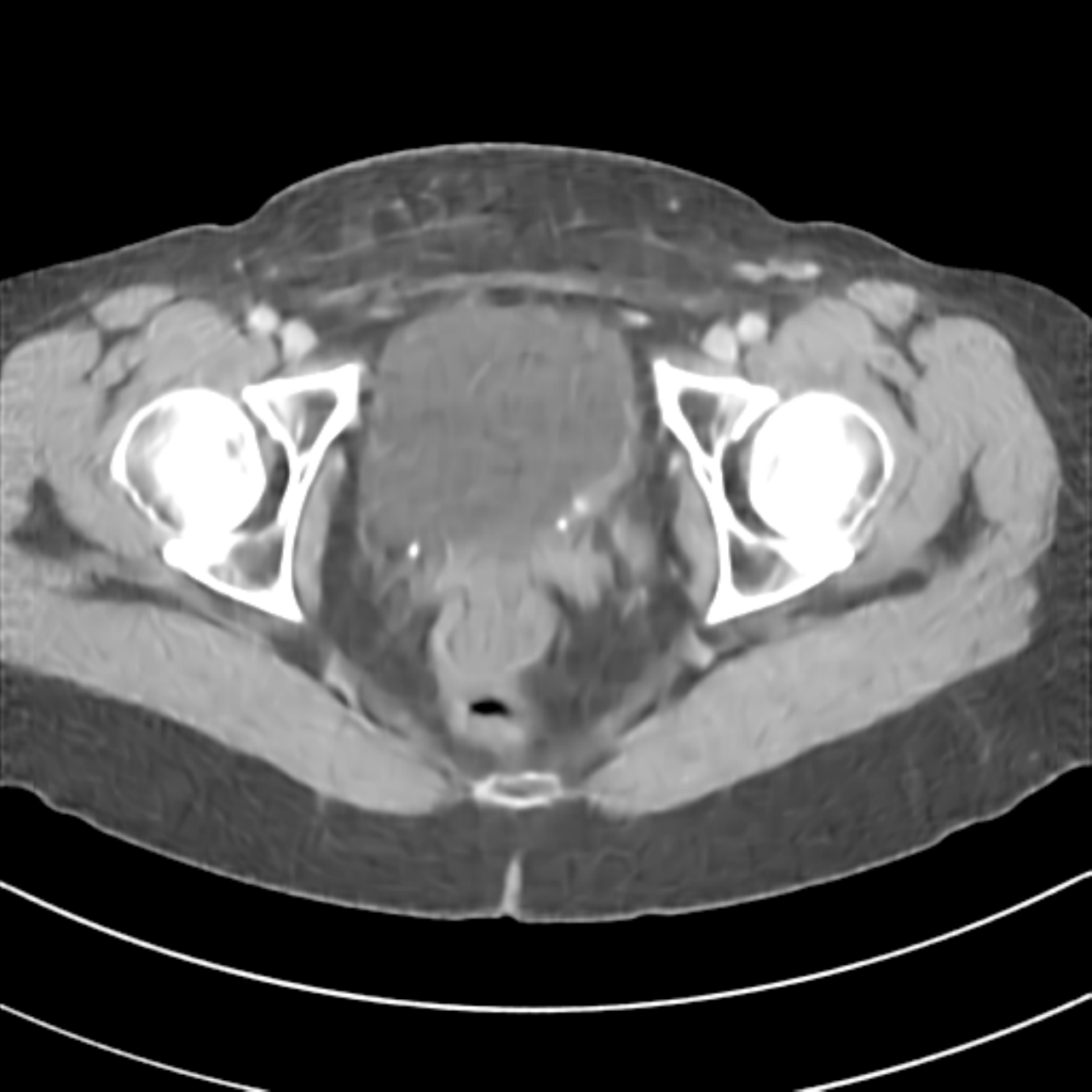}};
 \spy on (-0.55,1) in node [left] at%
   (-0.1\textwidth,-0.3\textwidth);
 \spy on (0.55,0.47) in node [left] at%
   (0.5\linewidth,-0.3\linewidth);   
\end{tikzpicture}
\caption{DL}
\end{subfigure}%
\caption{Example slice of the pelvis (level 0 HU, window 400 HU). Regions of interest highlight the structure in the fatty tissue (bottom left) and fine high-contrast structures inside the hip joint (bottom right).}
\label{fig:slice163}
\end{figure}

\begin{table}[]
    \centering
    \begin{tabular}{c c c c c c c c c}
         & \acs{FBP}  & \acs{TV} & \acs{TGV}  & Huber & \acs{DIP} & \acs{AR} & \acs{LPD} & \acs{DL} \\
    \hline 
    \acs{PSNR} &  40.95  & 46.37 & 46.70 &  46.70 & 46.29 & 46.96  & \textbf{49.68} & 48.20 \\
    \acs{SSIM} &  0.942  & 0.987 & 0.987 &  0.988 & 0.988 & 0.988 &  \textbf{0.992}  & 0.990
    \end{tabular}
    \caption{Performance metrics for various reconstruction methods in low-dose \ac{CT}. Note that \acs{LPD} requires supervised training data, whereas \acs{DL} can be trained against unsupervised data.}
    \label{tab:results}
\end{table}

\section{Conclusion}
\label{sec:conclusion}
In this work we have presented a novel view point on the dictionary learning, by showing that it maximizes the evidence lower bound similarly to variational auto-encoders. Moreover, we have shown that dictionaries can be successfully learned using optimization techniques for training neural networks. We verify that regularization learned dictionaries constitutes a powerful method for tomographic reconstruction, which is unsupervised with respect to tomographic data. We compare our method with many different model-based and data-driven approaches and conclude that it outperforms most of the methods by a large margin, while being inferior only to the learned primal-dual, which is a fully supervised deep-learning method. 

Even though dictionary learning based reconstruction leads to high quality images in terms of \ac{PSNR}, we observe that this method suffers from high frequency noise patterns shaped by dictionary atoms. On the other hand, the methods that rely on multi-layer convolutional neural networks, such as LPD and DIP, produce noise patterns with lower frequency. Therefore, we suppose that a more complex generative model should be able to provide further improvements in image quality.

Finally, all evaluated variational methods, including our dictionary learning based approach, required many iterations to converge. This compromises their applicability in practice. Addressing this problem by learned optimization \cite{banert2021accelerated} is another direction for future research.

\section*{Acknowledgments}
Jevgenija Rudzusika was supported by the Swedish Foundation of
Strategic Research grant AM13-0049, grant from the VINNOVA Open Innovation Hub project 2015-06759, and by Philips Healthcare.
Finally, Ozan Öktem was supported by the Swedish Foundation of
Strategic Research grant AM13-0049.

\printbibliography
\newpage

\appendix

\section{Comparison of patch-based and a convolutional synthesis operators}\label{app:CompPatch}

We compare the performance of patch-based synthesis operator $\SynthesisOp^p_{\hat{\Dict}}$ and the convolutional operator $\SynthesisOp_{\hat{\Dict}}$ in the case of tomographic reconstruction problem \cref{eq:RecoProblemGeneric}. 
This problem is reduced to \cref{eq:RecoProblem} in the case of convolutional synthesis operator. 
A similar formulation can be used when the linear synthesis operator $\SynthesisOp^p_{\hat{\Dict}}$ acts on patches: 
\begin{equation}\label{eq:RecoProblemPatch}
\begin{split}
    (\estim{\signal}, \estim{\dictcoeff}) 
    &\in\argmax_{(\signal, \dictcoeff)\in \RecSpace \times \DictCoeffSet} \DataLogLikelihood\bigl(\ForwardOp (\signal),  \data\bigr) + \frac{\regparamrec_1}{N_p} \sum_{i=1}^{N_p} \bigl\| \signal_i -\SynthesisOp^p_{\estim{\Dict}}(\dictcoeff_i) \bigr\|_2^2 +  \frac{\regparamrec_2}{N_p} \sum_{i=1}^{N_p} \|\dictcoeff_i\|_1
\end{split}
\end{equation}
Here $\{\signal_i\}_{i=1}^{N_p}, 
\signal_i \in \mathbb{R}^{k \times k}$ is a sequence of all overlapping $k \times k$ image patches, $\{\dictcoeff_i\}_{i=1}^{N_p} \subset \mathbb{R}^m$ is the sequence of corresponding coefficients with $m$ denoting the number of dictionary atoms. 
Finally, $N_p$ is the number of patches. 
In addition, we re-scale by a factor $k^{-2}$ simply to make values of regularization parameters $\regparamrec_1$ and $\regparamrec_2$ numerically closer to the values used in the convolutional setting \cref{eq:RecoProblem}.

On our validation set (21 image slices) we observe \ac{PSNR} values 47.94 and 48.22 for $\SynthesisOp^p_{\hat{\Dict}}$ and $\SynthesisOp_{\hat{\Dict}}$ respectively. In the case of patch-based synthesis, the best performance is achieved for $\regparamrec_1 = 10, \regparamrec_2 = 0.0006$.  We show example reconstructions obtained with different values of regularization parameters in the second row of \cref{fig:patch_based}. The best parameter setting in the patch-based case is highlighted in bold. Even in this setting, the image (\cref{fig:patch_based:d}) looks more noisy than the image synthesised with convolutional operator (\cref{fig:patch_based:b}).

We also tried to adapt the sparsity level during dictionary learning, to learn a better dictionary for the patch-based setting. However, this did not lead to significant improvements (best achieved \ac{PSNR} is 48.00).

\begin{figure}
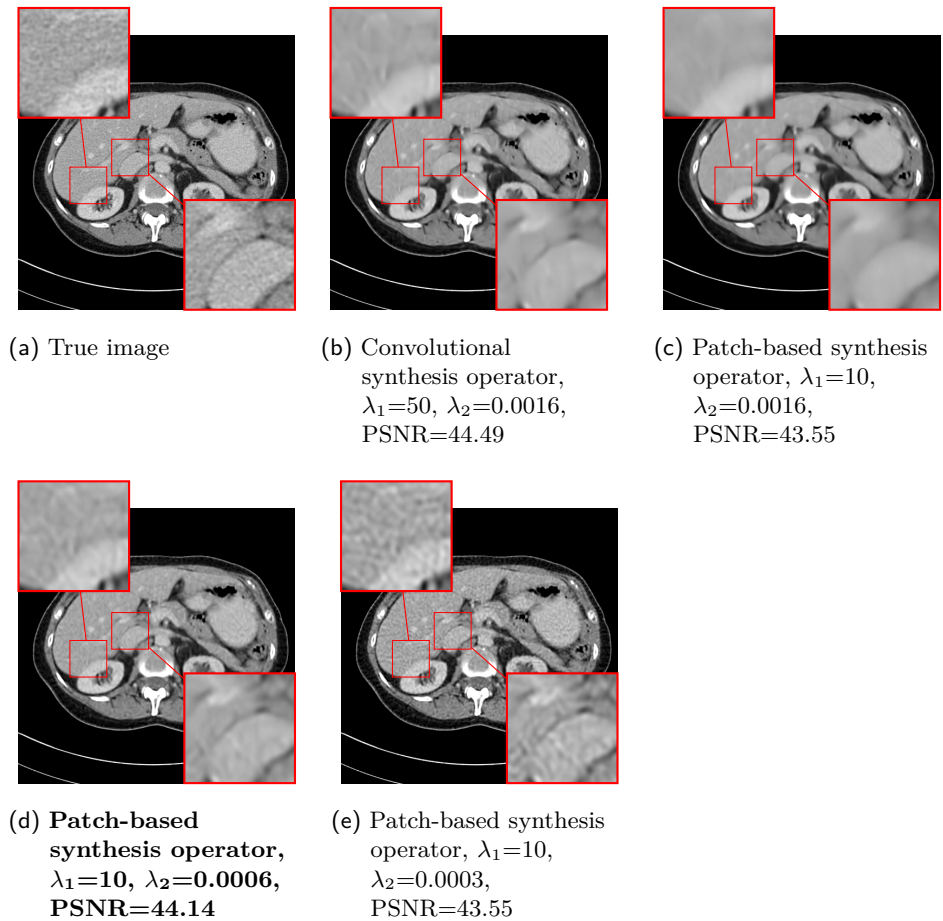

\centering
\begin{subfigure}[t]{0.3\textwidth}
\begin{tikzpicture}[spy using outlines={
  rectangle, 
  red, 
  magnification=3,
  size=0.4\linewidth, 
  connect spies}]
 \node{\includegraphics[width=\linewidth]{images/True_36}};
 \spy on (-0.35,0.2) in node [left] at%
   (0.5\linewidth,-0.3\linewidth);
 \spy on (-0.9,-0.17) in node [left] at%
   (-0.1\linewidth,0.4\linewidth);   
\end{tikzpicture}
\caption{True image}
\end{subfigure}%
\hspace{1em}
\begin{subfigure}[t]{0.3\textwidth}
\begin{tikzpicture}[spy using outlines={
  rectangle, 
  red, 
  magnification=3,
  size=0.4\linewidth, 
  connect spies}]
 \node{\includegraphics[width=\linewidth]{images/rf=0.0016_36}};
 \spy on (-0.35,0.2) in node [left] at%
   (0.5\linewidth,-0.3\linewidth);
 \spy on (-0.9,-0.17) in node [left] at%
   (-0.1\linewidth,0.4\linewidth);   
\end{tikzpicture}
\caption{Convolutional synthesis operator, $\regparamrec_1$=50, $\regparamrec_2$=0.0016, \ac{PSNR}=44.49}
\label{fig:patch_based:b}
\end{subfigure}%
\hfill
\begin{subfigure}[t]{0.3\textwidth}
\begin{tikzpicture}[spy using outlines={
  rectangle, 
  red, 
  magnification=3,
  size=0.4\linewidth, 
  connect spies}]
 \node{\includegraphics[width=\linewidth]{images/indep_rf=0.0016_36}};
 \spy on (-0.35,0.2) in node [left] at%
   (0.5\linewidth,-0.3\linewidth);
 \spy on (-0.9,-0.17) in node [left] at%
   (-0.1\linewidth,0.4\linewidth);   
\end{tikzpicture}
\caption{Patch-based synthesis operator, $\regparamrec_1$=10, $\regparamrec_2$=0.0016, \ac{PSNR}=43.55}
\end{subfigure}%
\\[1em]
\begin{subfigure}[t]{0.3\textwidth}
\begin{tikzpicture}[spy using outlines={
  rectangle, 
  red, 
  magnification=3,
  size=0.4\linewidth, 
  connect spies}]
 \node{\includegraphics[width=\linewidth]{images/indep_rf=0.0006_36}};
 \spy on (-0.35,0.2) in node [left] at%
   (0.5\linewidth,-0.3\linewidth);
 \spy on (-0.9,-0.17) in node [left] at%
   (-0.1\linewidth,0.4\linewidth);   
\end{tikzpicture}
\caption{\textbf{Patch-based synthesis operator, $\mathbf{\regparamrec_1}$=10, $\mathbf{\regparamrec_2}$=0.0006, \ac{PSNR}=44.14}}
\label{fig:patch_based:d}
\end{subfigure}%
\hfill
\begin{subfigure}[t]{0.3\textwidth}
\begin{tikzpicture}[spy using outlines={
  rectangle, 
  red, 
  magnification=3,
  size=0.4\linewidth, 
  connect spies}]
 \node{\includegraphics[width=\linewidth]{images/indep_rf=0.0003_36}};
 \spy on (-0.35,0.2) in node [left] at%
   (0.5\linewidth,-0.3\linewidth);
 \spy on (-0.9,-0.17) in node [left] at%
   (-0.1\linewidth,0.4\linewidth);   
\end{tikzpicture}
\caption{Patch-based synthesis operator, $\regparamrec_1$=10, $\regparamrec_2$=0.0003, \ac{PSNR}=43.55}
\end{subfigure}%
\hfill
\hspace{0.3\textwidth}
\caption{Example slice of the abdomen (level 80 HU, window 370 HU) reconstructed with  dictionary learning based regularization of image patches and different values of regularization parameters. \Acp{ROI} illustrate noise texture (top left) and delineation of low contrast structures (liver-kidney border in the top left \ac{ROI} and liver-pancreas transition as well as contrasted vessel in the liver in the bottom right \ac{ROI}.}
\label{fig:patch_based}
\end{figure}

\section{Learned dictionary}
\label{app:dic_viz}

\Cref{fig:learned_dics} shows learned dictionary atoms ordered by significance. 
We measure the significance of each atom by summing absolute values of the corresponding dictionary coefficients. The coefficients are calculated by solving the reconstruction problem in \cref{eq:RecoProblem} for all images in the validation set. 

Note that the learned dictionary also contains very general atoms that contain isolated point-like image features (bottom row in \cref{fig:learned_dics}). 
However, those atoms are used less significantly during reconstruction.
\begin{figure}
    \centering
    \includegraphics[width=\textwidth]{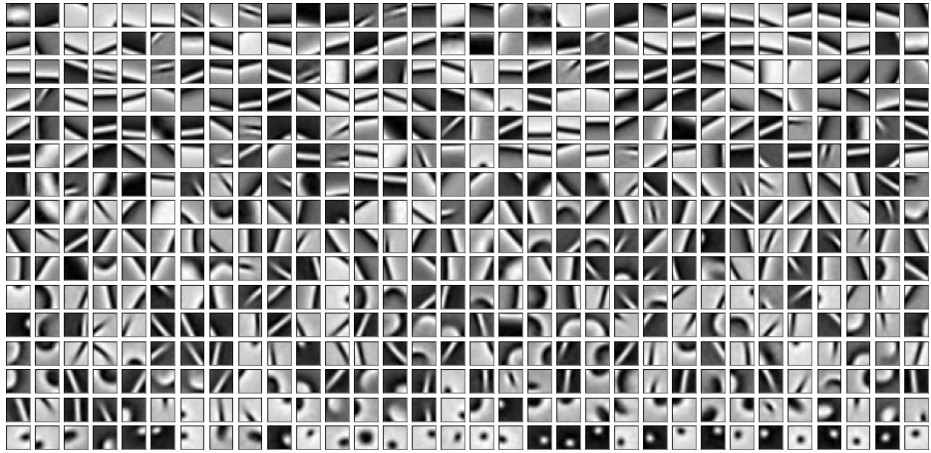}
    \caption{Learned dictionary atoms ordered by significance (from top to bottom).}
    \label{fig:learned_dics}
\end{figure}

\section{Quality of images depending on regularization parameters}

We show the dependency of dictionary learning based regularization \cref{eq:RecoProblem} on regularization parameters $\regparamrec_1$ and $\regparamrec_2$ in \cref{fig:lambda_dependency}. The optimal setting $\regparamrec_1 = 50$ and $\regparamrec_2 = 0.0016$ is highlighted in bold. 

The first parameter $\regparamrec_1$ controls the distance of reconstructed image $\signal$ from the image synthesised by the dictionary $\SynthesisOp_{\estim{\Dict}}(\dictcoeff)$. We show images for smaller values of $\regparamrec_1$ in the first row of \cref{fig:lambda_dependency}. We find that a wide range of values leads to similar results. For instance, for  $\regparamrec_1 = 50$ and $\regparamrec_1 = 10$ images are very similar. Only for much lower values, like $\regparamrec_1 = 1$ we can see that the noise pattern changes significantly. In particular, the low frequency noise shaped by the dictionary atoms becomes masked by high frequency noise. Unfortunately, quality of the image measured by \ac{PSNR} significantly decreases.

The second parameter $\regparamrec_2$ controls sparsity of dictionary coefficients. 
We show images for different values of this parameter in the second row of \cref{fig:lambda_dependency}. We can see that for a higher value ($\regparamrec_2=0.0024$) image is over-smoothed and for a lower value ($\regparamrec_2=0.0012$) image is more noisy. Nevertheless, quality of the image measured by \ac{PSNR} stays relatively high in both cases.
 
To summarize, the method seems robust to the choice of regularization parameter.

\begin{figure}
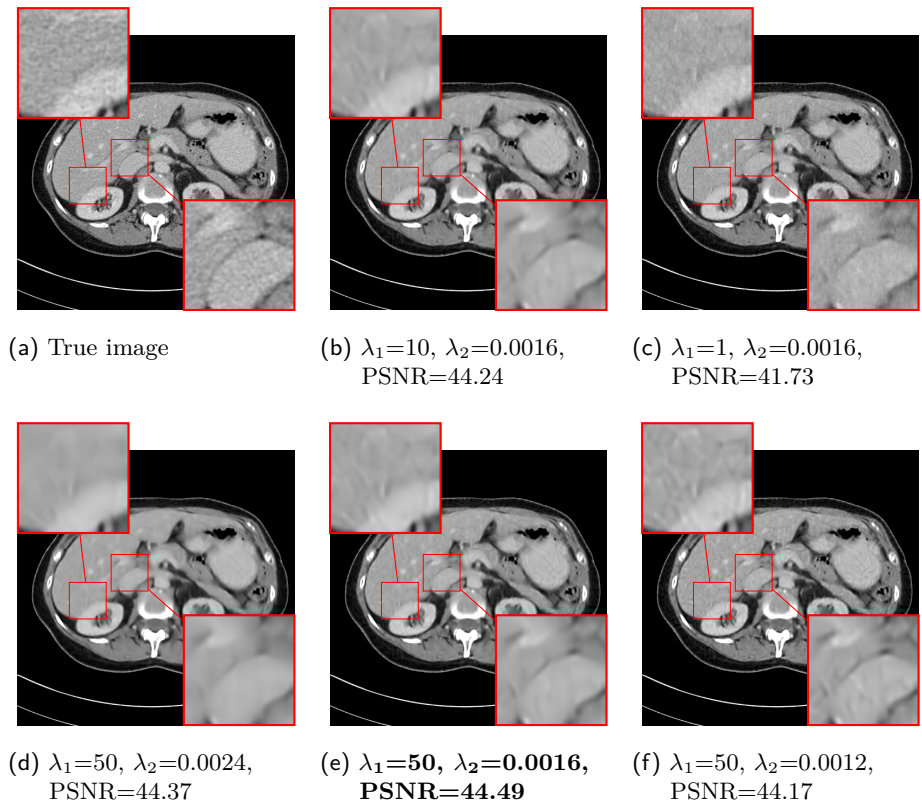

\centering
\begin{subfigure}[t]{0.3\linewidth}
\begin{tikzpicture}[spy using outlines={
  rectangle, 
  red, 
  magnification=3,
  size=0.4\linewidth, 
  connect spies}]
 \node{\includegraphics[width=\linewidth]{images/True_36}};
 \spy on (-0.35,0.2) in node [left] at%
   (0.5\linewidth,-0.3\linewidth);
 \spy on (-0.9,-0.17) in node [left] at%
   (-0.1\linewidth,0.4\linewidth);   
\end{tikzpicture}
\caption{True image}
\end{subfigure}%
\hspace{1em}
\begin{subfigure}[t]{0.3\textwidth}
\begin{tikzpicture}[spy using outlines={
  rectangle, 
  red, 
  magnification=3,
  size=0.4\linewidth, 
  connect spies}]
 \node{\includegraphics[width=\linewidth]{images/rf0=10_36}};
 \spy on (-0.35,0.2) in node [left] at%
   (0.5\linewidth,-0.3\linewidth);
 \spy on (-0.9,-0.17) in node [left] at%
   (-0.1\linewidth,0.4\linewidth);   
\end{tikzpicture}
\caption{$\regparamrec_1$=10, $\regparamrec_2$=0.0016, \ac{PSNR}=44.24}
\end{subfigure}%
\hspace{1em}
\begin{subfigure}[t]{0.3\textwidth}
\begin{tikzpicture}[spy using outlines={
  rectangle, 
  red, 
  magnification=3,
  size=0.4\linewidth, 
  connect spies}]
 \node{\includegraphics[width=\linewidth]{images/rf0=1_36}};
 \spy on (-0.35,0.2) in node [left] at%
   (0.5\linewidth,-0.3\linewidth);
 \spy on (-0.9,-0.17) in node [left] at%
   (-0.1\linewidth,0.4\linewidth);   
\end{tikzpicture}
\caption{$\regparamrec_1$=1, $\regparamrec_2$=0.0016, \ac{PSNR}=41.73}
\end{subfigure}%
\\[1em]
\begin{subfigure}[t]{0.3\textwidth}
\begin{tikzpicture}[spy using outlines={
  rectangle, 
  red, 
  magnification=3,
  size=0.4\linewidth, 
  connect spies}]
 \node{\includegraphics[width=\linewidth]{images/rf=0.0024_36}};
 \spy on (-0.35,0.2) in node [left] at%
   (0.5\linewidth,-0.3\linewidth);
 \spy on (-0.9,-0.17) in node [left] at%
   (-0.1\linewidth,0.4\linewidth);   
\end{tikzpicture}
\caption{$\regparamrec_1$=50, $\regparamrec_2$=0.0024, \ac{PSNR}=44.37}
\end{subfigure}%
\hspace{1em}
\begin{subfigure}[t]{0.3\textwidth}
\begin{tikzpicture}[spy using outlines={
  rectangle, 
  red, 
  magnification=3,
  size=0.4\linewidth, 
  connect spies}]
 \node{\includegraphics[width=\linewidth]{images/rf=0.0016_36}};
 \spy on (-0.35,0.2) in node [left] at%
   (0.5\linewidth,-0.3\linewidth);
 \spy on (-0.9,-0.17) in node [left] at%
   (-0.1\linewidth,0.4\linewidth);   
\end{tikzpicture}
\caption{\textbf{$\mathbf{\regparamrec_1}$=50, $\mathbf{\regparamrec_2}$=0.0016, \ac{PSNR}=44.49}}
\end{subfigure}%
\hspace{1em}
\begin{subfigure}[t]{0.3\textwidth}
\begin{tikzpicture}[spy using outlines={
  rectangle, 
  red, 
  magnification=3,
  size=0.4\linewidth, 
  connect spies}]
 \node{\includegraphics[width=\linewidth]{images/rf=0.0012_36}};
 \spy on (-0.35,0.2) in node [left] at%
   (0.5\linewidth,-0.3\linewidth);
 \spy on (-0.9,-0.17) in node [left] at%
   (-0.1\linewidth,0.4\linewidth);   
\end{tikzpicture}
\caption{$\regparamrec_1$=50, $\regparamrec_2$=0.0012, \ac{PSNR}=44.17}
\end{subfigure}%

\caption{Example slice of the abdomen (level 80 HU, window 370 HU) reconstructed with dictionary learning based regularization and different values of regularization parameters. \Acp{ROI} illustrate noise texture (top left) and delineation of low contrast structures (liver-kidney border in the top left \ac{ROI} and liver-pancreas transition as well as contrasted vessel in the liver in the bottom right \ac{ROI}.}
\label{fig:lambda_dependency}
\end{figure}

\end{document}